\documentclass[11pt]{article}
\usepackage[margin=1in]{geometry}

\usepackage{graphicx} 
\usepackage{amsmath} 
\usepackage{amssymb}  
\usepackage{amsthm}

\usepackage{color}
\allowdisplaybreaks[1]

\usepackage{float}
\usepackage{mwe}
\usepackage{subfig}
\usepackage[latin1]{inputenc}
\usepackage{enumerate}
\usepackage[shortlabels]{enumitem}

\usepackage{algorithm}
\usepackage{algpseudocode}
\usepackage{hyperref} 

\theoremstyle{plain}  
\newtheorem{theorem}{Theorem}[section]

\newtheorem{lemma}[theorem]{Lemma}
\newtheorem{corollary}[theorem]{Corollary}
\newtheorem{proposition}[theorem]{Proposition}

\theoremstyle{definition}

\theoremstyle{remark} 

\newtheorem{remark}{Remark}[section]
\newtheorem{definition}{Definition}

\newcommand{\poo}{p_{00}}
\newcommand{\qio}{q_{10}}
\newcommand{\poi}{p_{01}}
\newcommand{\pio}{p_{10}}
\newcommand{\pii}{p_{11}}
\newcommand{\qoo}{q_{00}}
\newcommand{\qoi}{q_{01}}
\newcommand{\qii}{q_{11}}
\newcommand{\uoo}{u^{00}}
\newcommand{\uoi}{u^{01}}
\newcommand{\uio}{u^{10}}
\newcommand{\uii}{u^{11}}

\newcommand{\p}{\mathbb{P}}
\newcommand{\E}{\mathbb{E}}

\long\def\gh#1{{\color{black}#1}}
\long\def\ghh#1{{\color{black}#1}}

\def\0{{\bf 0}}
\def\1{{\bf 1}}

\def \bmat{\left[\begin{matrix}}
	\def \emat{\end{matrix}\right]}

\def \xy1vec{\left[\begin{matrix}x\\y\\1\end{matrix}\right]}
\def \QED{\begin{flushright}\Halmos\end{flushright}\end{proof}}

\long\def\old#1{}

\definecolor{DarkerGreen}{RGB}{0,170,0}

\definecolor{orange}{rgb}{1,0.5,0}

\renewcommand\and{\end{tabular}\kern-\tabcolsep\ and\ \kern-\tabcolsep\begin{tabular}[t]{c}}
\let\origthanks\thanks
\renewcommand\thanks[1]{\begingroup\let\rlap\relax\origthanks{#1}\endgroup}

\title{\LARGE \bf Partial Recovery in the Graph Alignment Problem}
\author{Georgina Hall\thanks{Georgina Hall is with the department of Decision Sciences at INSEAD. Email: \texttt{georgina.hall@insead.edu}} \and Laurent Massouli\'e\thanks{Laurent Massouli\'e is with the DYOGENE Group at INRIA and the MSR-Inria Joint Centre. Email: \texttt{laurent.massoulie@inria.fr}}}

\begin{document}
\date{}
\maketitle

\begin{abstract}
In this paper, we consider the graph alignment problem, which is the problem of recovering, given two graphs, a one-to-one mapping between nodes that maximizes edge overlap. This problem can be viewed as a noisy version of the well-known graph isomorphism problem and appears in many applications, including social network deanonymization and cellular biology. Our focus here is on \emph{partial recovery}, i.e., we look for a one-to-one mapping which is correct on a fraction of the nodes of the graph rather than on all of them, and we assume that the two input graphs to the problem are correlated Erd\H{o}s-R\'enyi graphs of parameters $(n,q,s)$. Our main contribution is then to give necessary and sufficient conditions on $(n,q,s)$ under which partial recovery is possible with high probability as the number of nodes $n$ goes to infinity. In particular, we show that it is possible to achieve partial recovery in the $nqs=\Theta(1)$ regime under certain additional assumptions.
\end{abstract}



\section{Introduction}

The goal of \emph{graph alignment} is to infer, given two graphs, a one-to-one mapping between the nodes of the two graphs such that the number of edge overlaps is maximized. Graph alignment problems appear in a variety of different applications, including computer vision \cite{conte2004thirty}, natural language processing \cite{maccartney2008phrase,lacoste2006word}, biology \cite{elmsallati2015global}, and social network deanonymization \cite{narayanan2009anonymizing}. In biology, graph alignment is used for example to match proteins across species. This is generally difficult to do via a direct comparison of the amino acid sequences that make up the proteins as mutations of the species' genotype can lead to significant differences in the sequences. However, while their sequences may vary greatly, proteins generally retain similar functions in each species' metabolism. By drawing up a Protein-Protein Interaction graph for each species and aligning these graphs across species, biologists are able to match proteins to better understand, e.g., the evolution of protein complexes.
We refer the reader to \cite{elmsallati2015global} for more information on the topic. In the context of social networks, graph alignment has been successfully used to deanonymize private social networks. Arguably, the best-known example is that of the deanonymization of the Netflix Prize database using data from IMDb \cite{narayanan2006break} as it led to a court case against Netflix and the end of the Netflix Prize. Other instances of social network deanonymization via graph alignment have been investigated as well, and comprise e.g. the deanonymization of Twitter using Flickr \cite{narayanan2009anonymizing}. In both cases, the ideas are similar: if a user has an account on two social networks that serve similar purposes (e.g., both Netflix and IMDb enable movie reviewing), then the content shared on the two platforms by the user would be quite similar. In other words, if one was to create a graphical representation of the content shared on the platform by users, then these graphical representations would be very correlated across social networks. (Consider in the Netflix and IMDb case, a bipartite graph that links a user to a movie he or she has reviewed.) If one of the social platforms is public (as is IMDb) and the other, though being private, releases partial or anonymized information about its user content (as Netflix did during the Netflix Prize), then one can simply construct a similar graphical representation of the information, align the graphs obtained, and recover information on the private network from the public network, which leads to a breach of privacy of the private network's user's content.
 
 The set-up we consider here is a stylized mathematical model of problems of this type: we assume that the input graphs are generated via a probabilistic model. More specifically, these input graphs are obtained by constructing two correlated Erd\H{o}s-R\'enyi graphs with parameters $(n,q,s)$ (see Section \ref{sec:math.model} for a mathematical description of correlated Erd\H{o}s-R\'enyi graphs, roughly speaking, $n$ is the number of nodes of the graphs, $q$ is a measure of their sparsity, and $s$ describes how correlated they are), selecting one of the two graphs at random and randomly permuting its labels, keeping the other graph as is. The question we wish to answer is then: given two graphs generated thus, under what conditions on $(n,q,s)$ are we able to recover the underlying permutation? Or, to be more precise, under what conditions on $(n,q,s)$ are we able to \emph{partially} recover the underlying permutation, i.e., recover a permutation that overlaps with the permutation that was used on just a constant fraction $\alpha$ of its vertices? The notion of partial recovery is in contrast with the notions of exact (resp. almost exact) recovery which would aim to recover the correct permutation for all nodes (resp. all but a vanishing fraction of nodes) of the graph. While information-theoretical results such as ours exist in the literature for these two types of recovery, the notion of partial recovery has, to the best of our knowledge, not yet been considered (see Section \ref{sec:math.model} for a literature review). This is in spite of its relevance in applications: for example, it would be enough in the case of social network deanonymization to provide a partially recovered permutation to prove that a privacy breach has occurred. This paper aims to fill this gap in our collective knowledge. 
 
In it, we provide both necessary and sufficient conditions for partial recovery of a graph alignment between two correlated Erd\H{o}s-R\'enyi graphs. Our necessary condition states that if the difference between two probability distributions measured via the Kullback-Leibler divergence is too small---in fact, small in front of $1$---then one cannot hope to recover a partial alignment. The two probability distributions that are contrasted in the result are the joint distribution of two independent edges (one in each graph) and the joint distribution of two correlated edges (again, one in each graph). This result is consequently very easy to interpret: if two correlated edges too closely mimic in probabilistic terms two independent edges, then there is no hope to recover any meaningful alignment. For our sufficient condition, we show that if $nqs=\Theta(1)$ and some other conditions relating to sparsity of the graphs hold, then partial recovery is possible. Though this result is perhaps less intuitively interpreted, it has the advantage of closely paralleling existing information-theoretic results for other types of recovery. Roughly speaking, it is known that the regime $nqs=\log(n)+\omega(1)$ leads to exact recovery and that the regime $nqs=\omega(1)$ leads to almost exact recovery: our result goes a step further---we now know that the regime $nqs=\Theta(1)$ is a regime in which partial recovery is possible. Though our necessary and sufficient conditions are not always easy to compare, it should be noted that in the case where $s=s_0$ and $q=n^{-\beta}$ for some positive $\beta$, they nearly read as counterparts to one another (see Section \ref{subsec:discussion} for a more detailed analysis).

 Before moving onto the outline of the paper, we briefly comment on the proofs of our results to give the reader a flavor of the rest of the paper. The proof for necessity is based on a generalization of Fano's inequality. The proof for sufficiency involves designing and analyzing an algorithm which is guaranteed to return a permutation that overlaps with the true permutation on a fraction $\alpha$ of the nodes if our sufficient conditions hold. The algorithm we propose has some commonalities with that proposed in \cite{cullina2018partial} in the sense that it also involves the study of the degrees of the nodes of a particular intersection graph. However, our analysis is necessarily more involved in order to capture results in a regime ($nqs=\Theta(1)$) which \cite{cullina2018partial} does not cover.

The rest of the paper is structured thus: in Section \ref{sec:math.model}, we describe the mathematical formulation of our model, introduce our notation, present our main results, and discuss practical applications and future results. We also provide a literature review of existing work on the graph alignment problem and contrast these results to ours. In Section \ref{sec:imp}, we give the proof of validity of our necessary conditions and in Section \ref{sec:ub}, that of our sufficient conditions.

\section{Mathematical description of our problem} \label{sec:math.model}

The goal of this section is to set the stage for the rest of the paper. In Section \ref{subsec:GAP}, we present the mathematical formulation of the graph alignment problem and briefly review the literature on the topic. In Section \ref{sec:corr.ER}, we focus more specifically on the case where the two input graphs are assumed to be correlated Erd\H{o}s-R\'enyi graphs and review the literature that is specific to this case. In Section \ref{subsec:landau}, we provide a short reminder on Landau notation and, in Section \ref{subsec:results}, we give our main results. \gh{We end with Section \ref{subsec:discussion} which discusses the practical implications of our work, how the bounds compare, and future directions.}

\subsection{The graph alignment problem} \label{subsec:GAP}

Let $G_A=(V_A,E_A)$ and $G_B=(V_B,E_B)$ be two undirected unweighted graphs such that $|V_A|=|V_B|=n$. Recall that the adjacency matrix of a graph $G$ on $n$ vertices is an $n \times n$ binary matrix with entry $(i,j)$ equal to 1 if there is an edge between node $i$ and node $j$ in the graph, and $0$ otherwise. We denote by $A$ and $B$ the adjacency matrices of $G_A$ and $G_B$ respectively. The goal of graph alignment is to recover a permutation $\pi$ of the labels of $G_B$ such that the overlap between $G_A$ and $G_B$ is maximized. If $P_n$ is the set of permutations over the set $\{1,\ldots,n\}$, this can be written as the following optimization problem:
$$\max_{\pi \in P_n} \sum_{1 \leq i,j \leq n} A_{ij}B_{\pi(i)\pi(j)}+\sum_{1 \leq i,j \leq n}(1-A_{ij})(1-B_{\pi(i)\pi(j)}),$$
i.e., we maximize the number of overlapping edges and the number of non-overlapping edges between $G_A$ and the graph obtained by relabeling $G_B$. This optimization problem is equivalent to $\max_{\pi \in P_n} \sum_{1 \leq i,j \leq n} A_{ij}B_{\pi(i)\pi(j)}$, which corresponds to maximizing the number of overlapping edges only. Note that this equivalence is not true anymore when we consider a generalization of the previous model to a case where the number of nodes in $G_A$ and the number of nodes in $G_B$ are different, see \cite{feizi2019spectral}. One can rewrite the previous optimization problem in yet another way by using permutation matrices rather than permutations. Recall that if $\pi$ is a permutation over $\{1,\ldots,n\}$, then the corresponding permutation matrix $\Pi$ is an $n \times n$ binary matrix with entry $(i,j)$ equal to 1 if and only if $\pi(i)=j$. We use the convention throughout the paper that lower-case $\pi$ refers to the permutation, and upper-case $\Pi$ refers to the corresponding permutation matrix. Furthermore, we denote by $P_{n}^M$ the set of $n \times n$ permutation matrices. With this notation in mind, the graph alignment problem can be written
\begin{align}\label{eq:GAP}
\max_{\Pi \in P_{n}^M} \mbox{Tr}(A\Pi B \Pi^T).
\end{align}

This problem is related to a couple of other well-known combinatorial problems \cite{conte2004thirty}. The graph isomorphism problem, e.g., is a particular case of the graph alignment problem, where there exists a permutation matrix $\Pi$ such that $\mbox{Tr}(A\Pi B \Pi^T)=\mbox{Tr}(A^2)$, i.e., the graphs $G_A$ and $G_B$ overlap exactly up to a relabeling of the nodes of $G_B$. This problem has generated a lot of interest as it sits in a particular computational complexity class: it is not known if deciding whether two graphs are isomorphic is polynomial-time solvable or NP-hard. We refer the reader to \cite{read1977graph} for a survey of the topic. At the other end of the spectrum, the quadratic assignment problem (QAP) is a generalization of the graph alignment problem, where the matrices $A$ and $B$ in (\ref{eq:GAP}) are not limited to adjacency matrices of graphs \cite{rendl1994quadratic}. The quadratic assignment problem is known to be NP-hard to solve \cite{sahni1976p}. However, it does not follow from this result that the subcase of graph alignment problems is NP-hard to solve as we are imposing additional structure on the matrices $A$ and $B$. In fact, the instances that are shown to be hard to solve in \cite{sahni1976p} explicitly preclude the proof from transferring to the graph alignment setting as it relies on the fact that $A$ and $B$ are non-binary matrices. A different reduction based on the maximum common subgraph problem can be used however to show NP-hardness of the graph alignment problem \cite{bayati2013message}.

As a consequence of its intractability, heuristics for solving the graph alignment problem (i.e., problem (\ref{eq:GAP})) abound in the literature. These heuristics can be divided roughly into four categories: methods based on network topological similarity, message-passing methods, spectral methods, and convex relaxation-based methods. Some of these methods do have guarantees under the assumption that the input graphs are randomly generated following the correlated Erd\H{o}s-Renyi model, a model which we formally describe in Section \ref{sec:corr.ER}. We consequently revisit some of the literature we describe here in Section \ref{sec:corr.ER}. The methods based on network topological similarity are probably the most numerous and are especially prevalent in biology applications. They typically involve identifying and matching in both graphs either subgraphs of specific shapes (also known as \emph{graphlets}) or nodes with specific neighborhood structures. Algorithms of this type appear in \cite{kuchaiev2010topological,malod2015graal,zhang2010sapper,barak2018nearly,ding2018efficient}. A more exhaustive list can be found in \cite{conte2004thirty}. Message-passing methods include belief propagation approaches such as the one given in \cite{bayati2013message} as well as percolation-based methods which assume an initial seed set of matched nodes \cite{yartseva2013performance,kazemi2015growing,mossel2019seeded,dai2018performance}. For some of the latter approaches, the requirement that a seed set be given can be somewhat relaxed by considering a phase 1/phase 2-type algorithm. In the first phase, the algorithm infers from the graphs a reasonable seed set by using high-degree nodes; this seed set is then used to proceed with the algorithm in phase 2. Methods which use spectral properties of the graphs are also popular; see, e.g. \cite{singh2008global} which relies on a PageRank-style algorithm, or \cite{umeyama1988eigendecomposition,feizi2019spectral,fan2019spectral, fan2019spectral2} which focus on aligning the eigenvectors of the adjacency matrices of the graphs. Finally, some methods involve approximating the optimization problem given in (\ref{eq:GAP}) by a convex optimization problem. Indeed, (\ref{eq:GAP}) is equivalent to $\min_{\Pi \in P_{n}^M} ||A-\Pi B \Pi^T||_F^2,$ where $|| \cdot ||_F$ is the Frobenius norm. By invariance of the Frobenius norm with respect to orthogonal matrices, this can be written as 
$\min_{\Pi \in P_{n}^M} ||A\Pi-\Pi B||_F^2.$
A natural convex relaxation can then be obtained by replacing the set of permutation matrices by its convex hull, the set of doubly stochastic matrices, i.e., the set of matrices $$DS_n \mathrel{\mathop{:}}=\{D \in \mathbb{R}^{n \times n} ~|~ D1=1, D^T1=1, D \geq 0\},$$ where $1$ is the all-ones vector. The problem then becomes $\min_{D \in DS_n} ||AD-D B||_F^2$,
 which is a quadratic program to solve. Unfortunately, the matrix obtained is rarely a permutation matrix---in fact, it can be shown \cite{lyzinski2016graph} that, under the correlated Erd\H{o}s-Renyi model that we describe in Section \ref{sec:corr.ER}, this is very unlikely to occur. A second step consequently involves the projection of the doubly stochastic matrix obtained to the set of permutation matrices, a projection that can be done in polynomial time, via, e.g., the Hungarian algorithm. Quadratic programming relaxations of this type are studied in \cite{fogel2013convex,vogelstein2015fast, kezurer2015tight,dym2017ds++,feizi2019spectral,fishkind2012seeded,zaslavskiy2008path}. These relaxations sometimes also involve a quadratic regularization term added to the objective, which corresponds to the Frobenius norm of the doubly stochastic matrix. An alternative convex relaxation for (\ref{eq:GAP}) which gives rise to a linear program rather than a quadratic program can be obtained by replacing the Frobenius norm squared by the 1-norm applied to the vectorized matrix, i.e., $$||A \Pi-\Pi B||_1=\sum_{i,j} |(A\Pi-\Pi B)_{ij}|.$$ This approach is much less prevalent in the literature and to the best of our knowledge only appears in \cite{almohamad1993linear}. There also exist convex relaxations of (\ref{eq:GAP}) based on semidefinite programming. These are relaxations that have more often than not been developed for the QAP and are then leveraged on this subcategory of problems \cite{schellewald2005probabilistic,zhao1998semidefinite}. Though they can be quite powerful, their main caveat remains their prohibitively high running-time.

The boundaries of these four classes of approaches for the graph alignment problem are of course quite porous with many links between them. For example, the topological similarity method described in \cite{ding2018efficient}, which involves matching nodes via the number of neighbors of their neighbors, can be viewed as a second-order approximation of the doubly stochastic relaxation. Indeed, it is known that two graphs $G_A$ and $G_B$ have the same \emph{ultimate degree sequence} if and only if there exists a doubly stochastic matrix $D$ such that $AD=DB$ \cite{scheinerman2011fractional}. As the ultimate degree sequence of a graph is the ultimate degree sequence of its nodes, and the ultimate degree sequence of a node is an infinite list comprising in first position the degree of the node, in second position the degree of its neighbors, in third position the degree of the neighbors of its neighbors, and so on, it is not difficult to see why the algorithm in \cite{ding2018efficient} is a second-order approximation of the doubly stochastic relaxation. Likewise, the spectral method presented in \cite{fan2019spectral2} can be viewed as a less-constrained version of the doubly stochastic relaxation; see \cite{fan2019spectral2} for details. Further note that we have focused above on methods that return graph alignments under the assumption that the graphs $G_A$ and $G_B$ are unweighted graphs on the same number of nodes. Extensions to the graph alignment problem which cover cases where the number of nodes in each graph are different \cite{kazemi2015can} or where the graphs are weighted \cite{shirani2018typicality} have also been developed.

\subsection{The correlated Erd\H{o}s-R\'enyi model} \label{sec:corr.ER}

The goal of this paper is not to study the model in its full generality, nor from a computational perspective. Rather, we focus on a particular random generative model for the input graphs, the correlated Erd\H{o}s-R\'enyi model \cite{pedarsani2011privacy}, which is very common in the literature, and wish to understand the conditions under which one can or cannot recover an alignment in this model from an information-theoretic perspective. 

We start this section with a description of the Erd\H{o}s-R\'enyi model. Let $n$ be an integer, and $q,s$ scalars in $[0,1]$ with $q<s$. Let $G$ be an Erd\H{o}s-R\'enyi graph with parameters $n$ and $q/s \leq 1$. In other words, $G$ is a graph on $n$ nodes, two nodes of which are independently connected with probability $q/s$. Two copies of this graph are created and each copy is subsampled with probability $s$, which means that each edge present in the copies is maintained with probability $s$ and deleted with probability $1-s$ independently and at random. We denote by $G_A$ and $G_{B'}$ the two graphs obtained and by $A$ and $B'$ their adjacency matrices. They are both Erd\H{o}s-R\'enyi graphs with parameter $q$, but entries $(i,j)$ of $A$ and $B'$ are correlated Bernoulli random variables with correlation $\frac{s-q}{1-q}$. This is what gives this generation process the name of correlated Erd\H{o}s-R\'enyi model. We then let $\pi^*$ be a permutation over $\{1,\ldots,n\}$ with associated permutation matrix $\Pi^*$. We obtain $G_B$ by permuting the labels of $G_{B'}$ with $\pi^*$, i.e., $B_{\pi^*(i)\pi^*(j)}=B'_{ij}$ or equivalently, $\Pi^*B\Pi^{*T}=B'$. The input to the graph alignment problem is then $G_A$ and $G_B$ and our goal is to recover $\pi^*$ (or equivalently $\Pi^*$) from these observed graphs. In this framework, there is a notion of ground truth, given by $\pi^*$, which we refer to as \emph{the} graph alignment in the rest of the paper. It is related to the definition given in Section \ref{subsec:GAP} (which is the one used in the absence of a generative model and a ground truth) in the sense that, in the correlated Erd\H{o}s-R\'enyi model, any matrix $\hat{\Pi}$ solution to (\ref{eq:GAP}) is a maximum a posteriori estimator of $\Pi^*$ provided that the correlation between edges is positive, i.e., $s>q$.

As mentioned previously, the parameter $s$ quantifies the amount of noise added to the model. If $s=1$, $G_A$ and $G_B$ are isomorphic. In this setting, one can recover exactly the permutation both information-theoretically and computationally if $\log(n)+\omega(1) \leq nq \leq n-\log(n)- \omega(1)$ \cite{bollobas1982distinguishing,wright1971graphs,czajka2008improved}. The proofs of this result rely on showing that the automorphism group of $G_A$ is trivial under these conditions.

In our case, we assume $s<1$ and wish to recover the permutation $\pi^*$ \emph{partially}. This means that we wish to find a permutation $\hat{\pi}$ such that the overlap between $\hat{\pi}$ and the permutation $\pi^*$, 
$$\beta(\hat{\pi},\pi^*) \mathrel{\mathop{:}}=\frac{1}{n} \sum_{i=1}^n \textbf{1}_{\pi(i)=\pi^*(i)}$$
is greater than some constant $\alpha>0$ with high probability as $n$ goes to infinity. We refer to such a permutation $\hat{\pi}$ as a partial alignment of the graphs $G_A$ and $G_B$, or alternatively we say that we have partially recovered an alignment of $G_A$ and $G_B$. This concept is defined as an analog to the concept of partial recovery in the stochastic block model for community detection \cite{abbe2017community}. As far as we know, this paper is the first to cover this notion of recovery for graph alignment problems in the correlated Erd\H{o}s-Renyi model from an information theoretic perspective (some computational results for sparse graphs can be found in \cite{ganassali2020tree}).

A vast majority of prior papers have instead focused on \emph{exact recovery}, i.e., the case where one wishes to find $\hat{\pi}$ such that $\p(\beta(\hat{\pi},\pi^*)=1)=1-o(1)$. It is known that exact recovery can be achieved if $nqs \geq \log(n) +\omega(1)$ together with some conditions on the sparsity of the graphs and the strength of correlation of the edges. A converse result which states that one cannot hope to exactly recover $\pi^*$ if $nqs \leq \log(n)-\omega(1)$ and the correlation between the edges is strictly less than one is also given \cite{cullina2016improved,cullina2017exact}. The algorithm that shows that exact recovery can be achieved under the aforementioned conditions is exponential time. Many papers have consequently focused on providing polynomial-time (or quasi-polynomial-time) algorithms that exactly recover $\pi^*$ under certain conditions on the parameters \cite{barak2018nearly,mossel2019seeded,dai2018performance,ding2018efficient}. \cite{barak2018nearly} propose a quasi-polynomial-time algorithm that exactly recovers $\pi^*$ under the assumption that $nq/s \in [n^{o(1)},n^{1/153}] \cup [n^{2/3};n^{1-\delta}]$ for small $\delta>0.$ Mossel and Xu provide instead a seeded algorithm: under the assumption that the seed set is large enough (i.e., the fraction of seeded notes is larger than $n^{-1/2+3\epsilon}$) for fixed $\epsilon>0$), then exact recovery is possible all the way down to the information-theoretic threshold (under the condition that $nq/s \leq n^{1/2+\epsilon}$). \cite{dai2018performance} extend an algorithm called \emph{canonical labeling}, which proved to be successful for graph isomorphism, to the graph alignment setting. The algorithm relies on a first step, which generates a surrogate of a seed set, to move on to the second set where these seeds are used to align the rest of the graph. Under the assumption that one is able to generate a sufficiently large set in the first step, then exact recovery can be achieved under the condition that $nqs \geq \omega(n^{4/5}\log(n)^{7/5})$ and $q(1-s) \leq o((qs)^5/\log(n)^6)$. Finally, \cite{ding2018efficient} analyze a polynomial-time algorithm based on neighbors of neighbors of nodes in the graph and are able to show its success when $nq =\Omega(\log(n)^2)$ and $1-s=O(1/\log(n)^2)$. \cite{fan2019spectral,fan2019spectral2} propose an algorithm called GRAMPA for graph alignment: in particular, they show that if two graphs are correlated Erd\H{o}s-R\'enyi graphs with edge correlation coefficient $1-\sigma^2$ and average degree at least $\text{polylog}(n)$, then GRAMPA exactly recovers the alignment whp when $\sigma \lesssim 1/\text{polylog}(n)$.

Another type of recovery that has been investigated in the literature is that of \emph{almost exact recovery}, i.e., the case where one wishes to recover $\hat{\pi}$ in such a way that $ P(\beta(\hat{\pi},\pi^*)=1-o(1))=1-o(1)$. This is a stronger notion of recovery than the one we are interested in though it also appears in the literature under the name of partial recovery \cite{cullina2018partial}. For the moment, only information-theoretic results have been shown for this case: using the notion of $k$-core alignments, \cite{cullina2018partial}  show that almost exact recovery is possible when $nqs \geq \omega(1)$, $\frac{q^2(1-s)^2}{qs(1-2q+qs)}+2q(1-s) \leq n^{-\Omega(1)}$ and $qs \leq \frac{1}{8e^3}$. A converse result in the same paper states that almost exact recovery is impossible when $nqs \leq O(1)$ (and correlation between the edges in the graphs is strictly less than one.) As this paper is the closest in the literature to ours, we further contrast our results to the ones appearing in \cite{cullina2018partial} in Section \ref{subsec:results}, where we state our main results.

\subsection{A reminder regarding Landau notation}\label{subsec:landau}

As Landau notation is prevalent in the rest of the paper, we briefly remind the reader of the conventions used in this notation. Let $f$ and $g$ be two functions of $n$. We say that $f(n) \sim g(n)$ if $\lim_{n \rightarrow \infty} \frac{f(n)}{g(n)}=1$. We further say that $f(n)=o(g(n))$ if $\lim_{n \rightarrow \infty} \frac{f(n)}{g(n)}=0$ and that $f(n)=\omega(g(n))$ if $\lim_{n\rightarrow \infty} \left|\frac{f(n)}{g(n)}\right|=\infty$. We finally say that $f(n)=O(g(n))$ (resp. $f(n)=\Omega(g(n))$) if there exists $k >0$ and $n_0 \in \mathbb{N}$ such that for all $n \geq n_0$, $|f(n)| \leq k \cdot g(n)$ (resp. $f(n) \geq k \cdot g(n)$), and that $f(n)=\Theta(g(n))$ if $f(n)=O(g(n))$ and $f(n)=\Omega(g(n))$.

\subsection{Main results}\label{subsec:results}

The focus of this paper is on the case of \emph{partial recovery}, i.e., for fixed $\alpha \in (0,1)$, we wish to recover $\hat{\pi}$ such that
$$\p(\beta(\hat{\pi},\pi^*)>\alpha)=1-o(1).$$

A running assumption throughout the paper is that $s>q$. This implies that the covariance of $(A_{ij},B'_{ij})$, which we denote by $\sigma^2$, is positive, or equivalently that the correlation between $A_{ij}$ and $B'_{ij}$, which we denote by $\rho$, is positive. Indeed,
\begin{align}\label{eq:cov.corr}
\sigma^2=q(s-q) \text{ and } \rho=\frac{q(s-q)}{q(1-q)}.
\end{align}

\subsubsection{Impossibility result.}\label{subsec:major.results.lb}

The statement of the theorem requires us to first define two probability distributions over $\{0,1\} \times \{0,1\}$, which will appear frequently in the rest of the paper. These are the probability distributions of two correlated edges $(A_{ij},B'_{ij})$ in $G_A$ and $G_{B'}$, and of two independent edges $(A_{ij},B'_{kl})$, $(i,j) \neq (k,l)$, in $G_A$ and $G_{B'}$. Let $P$ be the distribution of $(A_{ij},B'_{ij})$. We use the shorthand $p_{xy}$ to denote the probability that $A_{ij}=x$ and $B_{ij}'=y$. By definition of the correlated Erd\H{o}s-Renyi model, we have that 
\begin{align}\label{eq:def.p}
\poo=1-2q+qs, ~\poi=q(1-s),~\pio=q(1-s), ~\pii=qs.
\end{align}
Likewise, let $Q$ be the distribution of $(A_{ij},B_{kl}')$. Using similar notation, we have that:
\begin{align}\label{eq:def.q}
\qoo=1-2q+q^2, ~\qoi=q(1-q),~\qio=q(1-q),~\qii=q^2.
\end{align}
Intuitively, our success in finding a graph alignment should hinge on how similar these two distributions are: if they are too similar for example, then we cannot hope to recover the underlying graph alignment. In this direction we introduce a measure of disparity between probability distributions, known as the \emph{Kullback-Leibler} divergence. Recall that the Kullback-Leibler divergence between two discrete probability distributions $P$ and $Q$ over $\mathcal{X}$ is given by 
$$D(P||Q)=\sum_{x \in \mathcal{X}} P(x) \log \left( \frac{P(x)}{Q(x)}\right).$$
It is always nonnegative and equals zero when $P=Q$ almost everywhere. We are now ready to state our theorem.

\begin{theorem}\label{th:lower.bd}
	Let $G_A$ and $G_B$ be two graphs generated via the correlated Erd\H{o}s-Renyi model $(n,q,s)$ with $\pi^*$ selected uniformly at random. Let $\alpha \in (0,1)$, possibly dependent on $n$. If 
	$$\frac{D(P||Q)}{\alpha}= o\left(\frac{\log(n)}{n}\right)$$	then, for any algorithm which infers $\hat{\pi}$ from $G_A$ and $G_B$, we have	$\p[\beta(\hat{\pi},\pi^*)<\alpha]\rightarrow 1 \text{ when } n\rightarrow \infty.$ In other words, no learning algorithm will be able to infer, from $G_A$ and $G_B$, a permutation $\hat{\pi}$ that overlaps with $\pi^*$ on more than a fraction $\alpha$ of its vertices.
\end{theorem}
The proof of this result is given in Section \ref{sec:imp} and relies on a generalization of Fano's inequality. Theorem \ref{th:lower.bd} is easy to interpret: when $P$ and $Q$ are close, then $D(P||Q)$ is smaller, which makes this condition easier to meet. In other words, as $P$ and $Q$ get more and more similar, it becomes harder to recover the graph alignment. This is also true as $\alpha$ gets closer to $1$, i.e., when we are trying to recover a more significant fraction of the nodes. Theorem \ref{th:lower.bd} implies the following corollary, which enables comparisons with \cite{cullina2018partial} and Theorem \ref{th:ub}.

\gh{ \begin{corollary}\label{corr:lb.nqs}
Let $G_A$ and $G_B$ be two graphs generated via the correlated Erd\H{o}s-Renyi model $(n,q,s)$ with $\pi^*$ selected uniformly at random. Let $\alpha \in (0,1)$, possibly dependent on $n$. If 
\begin{align}\label{conditions.corr}	
q=\Omega(n^{-k}) \text{ for some $k>0$, } q=o(n^{-1/2}), \text{ and } nqs=o(1),
\end{align} 
then, for any algorithm which infers $\hat{\pi}$ from $G_A$ and $G_B$, we have	$\p[\beta(\hat{\pi},\pi^*)<\alpha]\rightarrow 1 \text{ when } n\rightarrow \infty.$
\end{corollary}
The proof of this result is also given in Section \ref{sec:imp}.
In this form, it is much simpler to compare the impossibility result to the one that features in \cite{cullina2018partial} for almost exact recovery. As a reminder, this is the closest result to ours in the literature, and it states that almost exact recovery is impossible when $nqs=O(1)$. Thus, when the conditions on $q$ are met, our result is more restrictive than in \cite{cullina2018partial}, but the conclusion we get is correspondingly stronger. In \cite{cullina2018partial}, one cannot hope to obtain a permutation that agrees with $\pi^*$ on all but a vanishing fraction of vertices; in our case, we also rule out the possibility of obtaining a permutation that agrees with $\pi^*$ on more than a fraction $\alpha$ of vertices.
}


\subsubsection{Possibility result}
The possibility result is given next.

\begin{theorem}\label{th:ub}
	Let $0<\alpha<1$. If, there exist $\beta>0$, $\gamma>0$ such that for large enough $n$,
	\begin{equation}\label{eq:conditions.ub}
\begin{aligned}
&nqs\geq \max \left\{20,84\cdot \log \left( \frac{2}{1-\alpha}\right),\frac{16}{\min(\gamma,\beta)(1-\alpha)}\right\}\\
&s > \frac{8}{1-\alpha} \cdot q\\
&\frac{2q(1-s^2)}{s} \leq n^{-\beta}\\
&qs \leq n^{-2\gamma}
\end{aligned}
\end{equation}
then there exists an (exponential-time) algorithm which returns a permutation $\hat{\pi}$ such that $$\p[\beta(\hat{\pi},\pi^*) \geq \alpha] \rightarrow 1 \text{ when }n\rightarrow \infty.$$
\end{theorem}
Crucially, our assumptions cover the case where $nqs=\Theta(1)$ as the case where $nqs=\omega(1)$ is already covered by \cite{cullina2018partial}. It is shown there that when $nqs=\omega(1)$, almost exact recovery is guaranteed to happen.

The complete proof of this result is given in Section \ref{sec:ub}. The algorithm we propose stems from the idea that if we were able to select $\pi^*$ correctly among all permutations, then the intersection graph we would observe between $G_A$ and $G_B$ permuted by $\pi^*$ would be an Erd\H{o}s-R\'enyi graph of parameters $(n,qs)$. Our choice of a permutation $\pi$ is then such that the intersection graph obtained with $\pi$ has some commonalities with an $ER(n,qs)$ graph. In particular, in our regime of interest $(nqs=\Theta(1))$, it is known that if $nqs$ is above some constant threshold, a positive fraction of the nodes in $ER(n,qs)$ (with this fraction increasing in $nqs$) has degree greater than or equal to $nqs/2$. As we are only interested in returning a permutation $\pi$ that overlaps with $\pi^*$ on a fraction $\alpha$ of the nodes, our algorithm will output a permutation which gives rise to an intersection graph where only a little more than a fraction $\alpha$ of the nodes have degree greater than or equal to $nqs/2$.

In spirit, our algorithm has many commonalities with the algorithm given in \cite{cullina2018partial}---it was in fact inspired by the authors' results. The main difference is of course the regime under consideration. In their case, they consider $nqs=\omega(1)$ and show that almost exact recovery can be achieved. In our case, we show that in the regime $nqs=\Theta(1)$, we can achieve partial recovery. Combined with \cite{cullina2017exact} which discusses exact recovery, this leads to a nearly complete picture of which type of recovery is possible in which regime. In terms of proof techniques, as our algorithm shares many commonalities with the algorithm given in \cite{cullina2018partial}, the key ideas behind our proof do overlap considerably with those in \cite{cullina2018partial}. There are two main differences. First, the change of regime requires a much finer analysis to show that the algorithm that we propose terminates. Second, we propose an alternative approach to analyzing the moment generating function of the sum of correlated random variables $\sum_{i\neq j} A_{ij}B_{\pi(i)\pi(j)}$, which eschews the many sets, graph constructions, and functions introduced in \cite{cullina2018partial}. Rather, it relies on elementary algebra and recurrent sequences, which we would argue are easier to apprehend and use. \ghh{In this spirit, we are interested to note that some of the computational techniques developed in this paper have come in useful in more recent publications. For example, Lemma~\ref{prop:arithm} and an analog to Lemma \ref{th:sequence} from this paper are leveraged by \cite{wu2021settling} in their proof of the positive part of their Theorem 4 establishing the sharp threshold for exact recovery}.

The conditions in Theorem \ref{th:ub} merit a brief discussion. The first condition is the key condition here as it confirms that partial recovery is feasible in the $nqs=\Theta(1)$ regime. It should be noted though that, as our goal was to establish this latter point, we did not make a particular effort to drive the constant that appears in this condition down to its smallest value. This could be the subject of future work. The second condition amounts to requiring a stronger correlation between pairs $(A_{ij},B'_{ij})$ than being simply positive: it is automatically satisfied if $q=o(s)$. The third and fourth conditions are sparsity conditions and it can be observed that there is a trade-off between how sparse the graph is and how large $nqs$ should be. It may be that these last conditions are artefacts of the analysis rather than the algorithm in itself: in other words, it may be possible to remove them with a more nuanced analysis.

\gh{
\subsection{Discussion and future directions} \label{subsec:discussion}

\paragraph{Comparison of the possibility and impossibility results in the partial alignment case.} We contrast the necessary and sufficient conditions obtained in this paper. To do so, we make use of Corollary \ref{corr:lb.nqs} and Theorem \ref{th:ub}, which are easier to compare. We remark that in the case where $n^{-k} \leq q \ll n^{-1/2}$ for some $k>0$, and conditions $s> \frac{8}{1-\alpha}q$ and $2q\frac{1-s^2}{s} \leq n^{-\beta}$ for some $\beta>0$ hold, the remaining conditions ($nqs$ larger than a constant and $nqs=o(1)$) nearly partition the space. 

\paragraph{Relevance of the results in practice.} As noted above, partial recovery is a weaker form of recovery than exact or almost exact recovery, as it returns an alignment which only partially coincides with the true alignment. The counterpart to this is that it can be achieved under weaker conditions, namely when $nqs >C$ where $C$ is some constant (together with some other more technical constraints as given in Theorem \ref{th:ub}). Two particular regimes, which are of possible interest in applications, are covered by these weaker conditions. The first one is the case where the graphs $G_A$ and $G_B$ are sparse (e.g., $q \sim 1/n$ and $s=\Theta(1)$). Sparse graphs are quite prevalent in real-world applications, particularly as models of social networks and recommender databases, and thus our results can be use to tackle problems on these graphs, such as e.g. the problem of anonymizing/de-anonymizing these structures. The second case corresponds to a setting where the graphs are denser (e.g., $q \sim \frac{\log(n)}{n}$) but we wish to allow for more noise (e.g.  $1-s \sim 1-\frac{1}{\log(n)}$). This can be relevant, for example, when aligning Protein-Protein Interaction graphs for species whose most recent common ancestor was more distant.

\paragraph{Future research directions.} As discussed above, there is a gap between the upper and lower bounds given in this paper. Obtaining matching upper and lower bounds would thus be of great future interest, as would be a polynomial-time algorithm which would recover a partial alignment under similar assumptions to the exponential one in this paper. Another interesting research direction would be to study how one can leverage $m$ correlated Erd\"os-Renyi graphs, rather than 2, to obtain a permutation $\pi$ which overlaps with $\pi^*$ on a larger fraction of nodes than what is currently guaranteed under the conditions of Theorem \ref{th:ub}. Finally, the techniques developed in this paper to analyze sums of correlated random variables may be more widely applicable.
	
}

\section{Impossibility result}\label{sec:imp}

In this section, we give the proof of Theorem \ref{th:lower.bd}. We start off with some preliminaries in Section \ref{subsec:prem.lb} before moving onto the core of the proof in Section \ref{subsec:proof.lb}.

\subsection{Preliminaries: Fano's generalized inequality and permutation results} \label{subsec:prem.lb} 

 The proof of the impossibility result relies on a generalization of Fano's inequality which appears in e.g. \cite{santhanam2012information}. The formulation we use here comes from \cite{bhm}. This generalization differs from the classical Fano's inequality in the sense that it uses a wider definition of what a successful learning algorithm constitutes. In the classical Fano's inequality, an algorithm is unsuccessful if there is one difference or more between the ground truth and the algorithm output. In its generalization, we allow for some amount of differences before classifying the algorithm as unsuccessful. 
 
 It should be noted that the presentation of the theorem in \cite{bhm} is tailored quite specifically to a privatization setting. We rewrite the theorem here in fuller generality. With the assumptions we take, the proof of the theorem that is given in \cite{bhm} goes through exactly and hence we do not repeat it here.
 
 \begin{theorem}[Lemma 20 in \cite{bhm}] \label{th:gen.Fano}
 	Let $\mathcal{H}$ be a hypothesis class of cardinality $M$ with $Z$ being a hypothesis drawn uniformly at random from $\mathcal{H}$. Let $\hat{X}$ and $\hat{Z} \in \mathcal{H}$ be random variables defined in such a way that 
 	$Z \longrightarrow \hat{X} \longrightarrow \hat{Z}$
 	constitutes a Markov Chain. We define $d:\mathcal{H} \times \mathcal{H} \rightarrow \mathbb{R}_+$ to be a distance function, and, for a given $d>0$, we say that a learner is successful if 
 	$d(Z,\hat{Z})<d.$
 	For any $h \in \mathcal{H}$, we further define $B_d(h)=\{h' \in \mathcal{H}~|~d(h,h') \leq d\}$ and we let $M_d=\max_{h \in \mathcal{H}} |B_d(h)|.$ Then, we have
 	$$\p[d(\hat{Z},Z)>d] \geq 1-\frac{I(Z;\hat{X})+1}{\log(M/M_d)},$$
 	where $I(Z;\hat{X})$ is the mutual information between the pair $Z$ and $\hat{X}$.
 	In other words, for any learning algorithm, the average error probability will be greater than $1-\frac{I(Z;\hat{X})+1}{\log(M/M_d)}$. 
 \end{theorem}
 
 It is quite straightforward to transpose this to our setting, noting that 
 $\pi^* \in \mathcal P_n \longrightarrow (A,B) \longrightarrow \hat{\pi}$
 is a Markov chain and that one can define a distance metric over $P_n \times P_n$ by taking
 $d(\pi,\pi^*)=1-\beta(\pi,\pi^*).$ Given $\alpha>0$, we then say that a learner is successful if $\beta(\pi,\pi^*)>\alpha$. We thus obtain the following corollary.
 
 \begin{corollary}\label{cor:gen.Fano}
 	Let $\pi^*$ be a permutation drawn uniformly from $P_n$ and let $G_A$ and $G_B$ be two graphs generated using $\pi^*$ as described in Section \ref{subsec:GAP}, with adjacency matrices $A$ and $B$. For any learning algorithm (i.e., any algorithm which infers from $A$ and $B$ a permutation $\hat{\pi}$), we have that
 	$$\p(\beta(\pi,\pi^*)<\alpha) \geq1-\frac{I(\pi^*; (A,B))+1}{\log(M/M_{\alpha})},$$
 	where $M=|P_n|$ and $M_{\alpha}=\max_{\pi^* \in P_n} |\{\pi~|~\beta(\pi,\pi^*) \geq \alpha\}|$.
 \end{corollary}

This corollary is the cornerstone of our proof. As can be seen from its statement, it involves the ratio $M/M_{\alpha}$. Before computing this quantity in Proposition \ref{lem:M.alpha}, we first remind the reader of a few properties of permutations, which will come in useful in the proof of Proposition \ref{lem:M.alpha}. 

Recall that the set of permutations over $\{1,\ldots,n\}$, denoted by $P_n$, is a group when equipped with the composition operation $\circ$. Its size is $n!$. As $P_n$ is a group, it follows that if $\pi_1$ and $\pi_2$ are in $P_n$, then $\pi_1 \circ \pi_2$ is again in $P_n$. Likewise, each permutation $\pi$ has an inverse permutation $\pi^{-1}$ which is such that the application of $\pi^{-1} \circ \pi$ (or $\pi \circ \pi^{-1}$) to $\{1,\ldots,n\}$ leaves it unchanged. A notion that will be key in Proposition \ref{lem:M.alpha} is that of \emph{rencontres numbers}. These are the number of permutations over $n$ elements that have $k$ fixed points, i.e., $k$ integers which are mapped to themselves. We refer to the rencontres numbers as $D_{n,k}$ and we use the fact that 
$$D_{n,k} \sim e^{-1} \cdot \frac{n!}{k!}.$$
We are now able to show Proposition \ref{lem:M.alpha}.

\begin{proposition}\label{lem:M.alpha}
	Let $M$ and $M_{\alpha}$ be defined as in Corollary \ref{cor:gen.Fano}. We have that:
	$\frac{M}{M_{\alpha}}=\Omega((n\alpha)!).$
	\end{proposition}

\begin{proof}{\gh{Proof.}}
	As $\pi^*$ and $\pi$ are in $P_n$, then $\tilde{p}\mathrel{\mathop{:}}=\pi^* \circ \pi$ is again a  permutation and the overlap between $\pi^*$ and $\pi$ can be characterized via the fixed points of $\tilde{p}$:
	\begin{align}\label{eq:charac.beta}
	\beta(\pi,\pi^*)=\frac{1}{n} \sum_{i=1}^n \textbf{1}_{\pi(i)=\pi^*(i)}=\frac{1}{n} \sum_{i=1}^n \textbf{1}_{\tilde{p}(i)=i},
	\end{align}
	i.e., the overlap between $\pi^*$ and $\pi$ is simply the number of fixed points of $\tilde{p}$ (divided by $n$). From this it follows that 
$$|\{\pi\in P_n~|~\beta(\pi,\pi^*) \geq \alpha\}|=|\{\tilde{p} \in P_n ~|~ \sum_{i=1}^n \textbf{1}_{\tilde{p}(i)=i} \geq \lceil n\alpha \rceil\}|=\sum_{k=\lceil n \alpha\rceil}^{n} D_{n,k}.$$
Note that this value does not change with $\pi^*$, and so 
$$M_{\alpha}=\sum_{k=\lceil n \alpha\rceil}^{n} D_{n,k}\sim \sum_{k=\lceil n \alpha\rceil}^{n} e^{-1} \frac{n!}{k!}.$$
Using Taylor's theorem on $e^1$ with e.g. the mean-value form of the remainder, we get that $$\sum_{k=\lceil n\alpha \rceil}^n \frac{1}{k!}=O\left( \frac{1}{\lceil n \alpha \rceil !} \right)=O\left( \frac{1}{ (n \alpha)! } \right)$$ which implies that $M_{\alpha}=O(n!/(n\alpha)!)$ and that $\frac{M}{M_{\alpha}}=\Omega((n\alpha)!)$.
\end{proof}

\subsection{Proof of Theorem \ref{th:lower.bd}}\label{subsec:proof.lb}

 We now give the proof of Theorem \ref{th:lower.bd}. At a high level, it involves tightly bounding each quantity that appears in Corollary \ref{cor:gen.Fano}.
 
 \begin{proof}{Proof of Theorem \ref{th:lower.bd}.}  	
 We use Corollary \ref{cor:gen.Fano} to find conditions on $n,q,$ and $s$ under which $\p(\beta(\pi,\pi^*)<\alpha)=1-o(1)$. From Proposition \ref{lem:M.alpha}, we have that
 $\frac{M}{M_{\alpha}}=\Omega((n\alpha)!)$ and so, using Stirling's formula, we deduce that  	$\log(M/M_{\alpha}) = \Omega (n\alpha \log(n\alpha)).$ 	It now remains to compute an upper bound on $I(\pi^*,(A,B))$ to conclude. \gh{Recall that $\pi^*$ is selected uniformly at random and that $G_A$ and $G_B$ are two graphs generated via the correlated Erd\"os-R\'enyi model $(n,q,s)$. Let $\tilde{\pi} \in P_n$ and $\tilde{A}, \tilde{B}$ be the adjacency matrices of two undirected graphs on $n$ nodes. Following standard information inequalities (see, e.g., \cite{cover2012elements}), we have:
 \begin{align*}
 I(\pi^*; (A,B))=D\left(\mathbb{P}(\pi^*=\tilde{\pi}, (A,B)=(\tilde{A},\tilde{B}))~||~\mathbb{P}(\pi^*=\tilde{\pi})\cdot \mathbb{P}\left((A,B)=(\tilde{A},\tilde{B})\right)\right).
 \end{align*}
 As 
 $\mathbb{P}\left(\pi^*=\tilde{\pi},(A,B)=(\tilde{A},\tilde{B})\right)=\mathbb{P}\left( (A,B)=(\tilde{A},\tilde{B})~|~ \pi^*=\tilde{\pi}\right) \cdot \mathbb{P}(\pi^*=\tilde{\pi}),$
 it follows that:
 \begin{small}
 \begin{align*}
 I(\pi^*; (A,B))=\sum_{\tilde{\pi}} \mathbb{P}(\pi^*=\tilde{\pi}) \cdot \sum_{(\tilde{A},\tilde{B})} \mathbb{P}\left( (A,B)=(\tilde{A},\tilde{B})~|~ \pi^*=\tilde{\pi}\right)  \log \left( \frac{\mathbb{P}\left( (A,B)=(\tilde{A},\tilde{B})~|~ \pi^*=\tilde{\pi}\right)}{\mathbb{P}\left((A,B)=(\tilde{A},\tilde{B})\right)}\right)
 \end{align*}
 \end{small}
Letting $\mathcal{P}$ be the joint distribution of $(A,B)$ under the correlated Erd\"os-R\'enyi model $(n,q,s)$ and $\mathcal{Q}$ be the joint distribution of adjacency matrices $(A,B)$ of two independent Erd\"os-R\'enyi graphs, $G_A$ and $G_B$, of parameters $(n,q)$, we write:
\begin{small}
\begin{align}
I(\pi^*; (A,B))&=\mathbb{E}_{\pi^*} \left[ \sum_{(\tilde{A},\tilde{B})} \mathbb{P}\left( (A,B)=(\tilde{A},\tilde{B})~|~ \pi^*=\tilde{\pi}\right) \cdot \log \left( \frac{\mathbb{P}\left( (A,B)=(\tilde{A},\tilde{B})~|~ \pi^*=\tilde{\pi}\right)}{\mathcal{Q}(\tilde{A},\tilde{B})}\cdot \frac{\mathcal{Q}(\tilde{A},\tilde{B})}{\mathcal{P}(\tilde{A},\tilde{B})}\right) \right] \nonumber \\
&=\mathbb{E}_{\pi^*} \left[ \sum_{(\tilde{A},\tilde{B})} \mathbb{P}\left( (A,B)=(\tilde{A},\tilde{B})~|~ \pi^*=\tilde{\pi}\right) \cdot \log \left( \frac{\mathbb{P}\left( (A,B)=(\tilde{A},\tilde{B})~|~ \pi^*=\tilde{\pi}\right)}{\mathcal{Q}(\tilde{A},\tilde{B})} \right) \right] \nonumber \\
&-\sum_{(\tilde{A},\tilde{B})} \mathcal{P}\left( \tilde{A},\tilde{B} \right) \cdot \log \left( \frac{\mathcal{P}(\tilde{A},\tilde{B})}{\mathcal{Q}(\tilde{A},\tilde{B})}\right) \nonumber \\
&\leq \mathbb{E}_{\pi^*} \left[ D\left( \mathbb{P}\left( (A,B)=(\tilde{A},\tilde{B})~|~ \pi^*=\tilde{\pi}\right) ~||~ \mathcal{Q}(\tilde{A},\tilde{B}) \right) \right], \label{eq:proof.imposs}
\end{align}
\end{small}
where we have used the total law of probability in the second equality and the fact that 
$$\sum_{(\tilde{A},\tilde{B})} \mathcal{P}\left( \tilde{A},\tilde{B} \right) \cdot \log \left( \frac{\mathcal{P}(\tilde{A},\tilde{B})}{\mathcal{Q}(\tilde{A},\tilde{B})}\right) =D(\mathcal{P}||\mathcal{Q}) \geq 0$$
in the inequality. Now, note that
\begin{align*}
\mathbb{P}\left( (A,B)=(\tilde{A},\tilde{B})~|~ \pi^*=\tilde{\pi}\right) &=\mathbb{P} \left( (A_{12},B_{\tilde{\pi}(1) \tilde{\pi}(2)})=(\tilde{A}_{12},\tilde{B}_{\tilde{\pi}(1)\tilde{\pi}(2)}), \ldots ~|~ \pi^*=\tilde{\pi}  \right)\\
&= P\left(\tilde{A}_{12},\tilde{B}_{\tilde{\pi}(1)\tilde{\pi}(2)}\right) \ldots P\left(\tilde{A}_{\tilde{\pi}(n-1)\tilde{\pi}(n)},\tilde{B}'_{\tilde{\pi}(n-1)\tilde{\pi}(n)}\right)
\end{align*}
and $\mathcal{Q}\left(\tilde{A},\tilde{B}\right)=Q\left(\tilde{A}_{12},\tilde{B}_{12}\right) \ldots Q\left(\tilde{A}_{(n-1)n},\tilde{B}_{(n-1)n}\right)$. Using these results in (\ref{eq:proof.imposs}), we obtain:
\begin{align*}
I(\pi^*; (A,B)) \leq \binom{n}{2} D(P||Q),
\end{align*}
from whence, $\p(\beta(\pi,\pi^*)<\alpha) \geq 1-\Omega \left(\frac{\binom{n}{2}D(P||Q)+1}{n\alpha \log(n\alpha)} \right).$
 	Hence, if $$\frac{D(P||Q)}{\alpha} = o \left(\frac{\log n}{n}\right),$$ it follows that $\p(\beta(\pi,\pi^*)<\alpha)=1-o(1)$.}
 \end{proof}

\gh{
\subsection{Proof of Corollary \ref{corr:lb.nqs}}
We provide a short proof of Corollary \ref{corr:lb.nqs} in this section.

\begin{proof}{Proof of Corollary \ref{corr:lb.nqs}.}
	We make use of the following reverse Pinsker inequality (see, e.g., \cite[Section 6.C]{sason2016f}): for two probability measures $P$ and $Q$  defined on a common finite set $\mathcal{A}$, and $Q$ is strictly positive on $\mathcal{A}$, which has more than one element, we have
	$$D(P||Q) \leq \log \left( \frac{1}{Q_{min}}\right) \cdot |P-Q|,$$
	where $|\cdot|$ is the total variation between $P$ and $Q$, and $Q_{min}=\min_{a \in \mathcal{A}}Q(a)\leq \frac12.$ In our setting, we have that $Q_{min}=q^2 \leq \frac12$ for large enough $n$, and 
	$|P-Q|=q(s-q).$
	We now show that, under the conditions given in (\ref{conditions.corr}), 
	$$\frac{D(P||Q)}{\alpha}=o \left( \frac{\log(n)}{n}\right),$$
	which concludes the proof. We have
	\begin{align*}
	\frac{D(P||Q)}{\alpha} \cdot \frac{n}{\log(n)} \leq \log \left( \frac{1}{Q_{min}}\right) \cdot |P-Q|\cdot \frac{n}{\alpha\log(n)}=\log \left( \frac{1}{q^2}\right) \cdot \frac{nqs-nq^2}{\alpha \log(n)}.
	\end{align*}
	It is easy to see using (\ref{conditions.corr}) that $\log \left( \frac{1}{q^2} \right)=O(\log(n))$, that $nqs=o(1)$, and that $nq^2=o(1)$. Thus, 
	$$\frac{D(P||Q)}{\alpha} \cdot \frac{n}{\log(n)} \leq O(\log(n)) \cdot \frac{o(1)}{\alpha \log(n)}=o(1).$$
\end{proof}

}

\section{Possibility result}\label{sec:ub}

In this section, we give a proof of Theorem \ref{th:ub}. To do this, we need to provide an algorithm that outputs a permutation $\pi$ which overlaps with $\pi^*$ over at least $n\alpha$ vertices whp. In our case, the algorithm involves enumerating each permutation and checking whether it is \emph{good} (see Definition \ref{def:good.perm}). We give the proof that this algorithm terminates whp (i.e., that there exists a permutation that is good whp) and that any good permutation must overlap with $\pi^*$ on at least $n\alpha$ vertices whp in Section \ref{subsec:proofs.ub}. Before getting to that however, we briefly discuss, in Section \ref{subsec:prelim.ub}, some properties of permutations and $k$-cores of graphs which will come in useful in the proofs of Section \ref{subsec:proofs.ub}.

\subsection{Preliminaries: $k$-cores and more permutation results}\label{subsec:prelim.ub}

\subsubsection{Definition and size of a $k$-core}

A $k$-core of a graph $G$ is the largest induced subgraph of $G$ with minimum degree at least $k$. (Note that it may be empty.) The $k$-core is unique and can simply be obtained by iteratively removing all nodes of degree less than $k$. It is easy to see from this definition that if a graph has a $k$-core of size $N$ (i.e., the $k$-core contains $N$ nodes) then $G$ has at least $N$ nodes of degree greater than or equal to $k$. Conditions under which a $k$-core in an Erd\H{o}s-R\'enyi$(n,p)$ graph is guaranteed to be non-empty are known: in fact, if $k\geq 3$ and the $k$-core is non-empty, then it contains a positive fraction of the vertices of the graph. We give the exact statement of this result here as it will be useful to us later on. 

\begin{lemma}[Theorem 2 in \cite{pittel1996sudden}]\label{lem:k,core}
	Let $\mu > 0$ and let $Po(\mu)$ denote a Poisson random variable of parameter $\mu$. We define 
	$$\psi_j(\mu)=\p(Po(\mu)\geq j)=\sum_{i=j}^{\infty}\frac{\mu^i e^{-\mu}}{i!}, \text{ for an integer }j\geq 0.$$
	We further let $k\geq 3$ and take
	$$c_k\mathrel{\mathop{:}}=\inf_{\mu>0}\frac{\mu}{\psi_{k-1}(\mu)}.$$
	For $\lambda \geq c_k$, define $\mu_k(\lambda)>0$ to be the largest solution to $\frac{\mu}{\psi_{k-1}(\mu)}=\lambda.$ Now, consider the random graph $G(n, \lambda/n)$ where $\lambda$ is fixed. If $\lambda>c_k$, then whp the $k$-core of $G$ contains a fraction $\psi_k(\mu_k(\lambda))$ of the nodes.
\end{lemma}

\subsubsection{Additional properties of permutations.}\label{subsec:perms}
As done in Section \ref{sec:imp}, we consider in this section the permutation $p\mathrel{\mathop{:}}=\pi \circ \pi^{*-1}$ and its corresponding permutation matrix $P\mathrel{\mathop{:}}=\Pi \cdot \Pi^{*T}$. Using this notation, we have that $B_{\pi(i)\pi(j)}=B'_{p(i)p(j)}$. Of interest to us will be two sets related to the fixed points of $p$.
\begin{definition}
		Let $\pi^*$ be a permutation in $P_n$ and let $\pi$ be a permutation such that $\beta(\pi,\pi^*)=\epsilon$. For $p=\pi \circ \pi^{*-1}$, define the sets:
	\begin{align}\label{def:sets}
	S\mathrel{\mathop{:}}=\{(i,j)\in \{1,\ldots,n\}^2~|~i\neq j\},~S_1\mathrel{\mathop{:}}=\{(i,j)\in S ~|~i=p(i)\}, \text{ and }S_2\mathrel{\mathop{:}}=S\backslash S_1
	\end{align}
	and let $|S|$, $|S_1|$ and $|S_2|$ denote their respective cardinalities. It is immediate to see that
	\begin{align*}
	|S|=n(n-1), \quad |S_1|=n\epsilon(n-1), \quad |S_2|=n(1-\epsilon)(n-1).
	\end{align*}
\end{definition}
Note that from these definitions, if $(i,j) \in S_2$, then it cannot be the case that $(i,j)=(p(i),p(j))$; however, $S_2$ does contain all pairs $(i,j)$ such that $(i,j)=(p(j),p(i))$. This leads us to further subdivide the set $S_2$: \gh{let 
$$S_2^1=\{(i,j)\in S_2~|~ (p(i),p(j))=(j,i)\} \text{ and } S_2^2=S_1 \backslash S_2^1.$$
We also introduce a set $\bar{S}_2^2$, containing \emph{unordered pairs},
$$\bar{S}_2^2\mathrel{\mathop{:}}=\left\{ \{i,j\}\in \{1,\ldots,n\}^2~|~ i\neq j,~ \{i,j\}\neq \{p(i),p(j)\}\right\},$$ 
and related to $S_2^2$ in the sense that $S_2^2 \subseteq \{ (i,j) \in S~|~ (i,j) \neq (p(i),p(j)),~(i,j) \neq (p(j),p(i))\},$ which is the ordered counterpart of $\bar{S}_2^2$. 
It is easy to see that 
\begin{align}\label{eq:size.S21.S22}
|S_2^1|\leq n(1-\epsilon) \text{ and }|\bar{S}_2^2|\leq \binom{n}{2}- \binom{n\epsilon}{2}\leq \frac{n^2}{2}(1-\epsilon^2).
\end{align}
We now define the concept of a \emph{cycle}, which will play an important role in Section \ref{subsec:proofs.ub}. A cycle $C$ is a set of unordered pairs $\{\{i,j\}\}$ such that if $\{i,j\} \in C$, for any other element $\{k,l\} \in C$, there exists $s \in \mathbb{N}$ such that $\{k,l\}=\{p^s(i), p^s(j)\}$. It is not hard to see that the cycle relationship between $\{i,j\}$ and $\{k,l\} \in \bar{S}_2^2$ defines an equivalence relationship over $\bar{S}_2^2$. Thus, $\bar{S}_2^2$ can be partitioned into cycles. It naturally follows that a pair $\{i,j\}$ appears in at most one cycle and only once in that cycle. Furthermore, by definition of $\bar{S}_2^2$, all cycles in $\bar{S}_2^2$ must be of length greater or equal to 2. With this in mind, we use the notation $\mathcal{C}_k$ to denote the set of cycles in $\bar{S}_2^2$ of size $k$ and $\ell_k$ to denote the number of cycles of size $k$ belonging to $\bar{S}_2^2$. We have $2\leq k \leq |\bar{S}|_2^2$,}
\begin{align}\label{eq:partition.S22}
\gh{\bar{S}_2^2=\cup_{k=2}^{|\bar{S}_2^2|} \mathcal{C}_k \text{ and } |\bar{S}_2^2|=\sum_{k=1}^{|\bar{S}_2^2|} k \cdot \ell_k.}
\end{align}

\subsection{Proof of Theorem \ref{th:ub}}\label{subsec:proofs.ub}

We prove Theorem \ref{th:ub} in two steps. First, we define the notion of a \emph{good permutation} and we show that under the conditions of Theorem \ref{th:ub}, whp $\pi^*$ is a good permutation (Proposition \ref{lem:alg.term}). The algorithm then involves enumerating all permutations until we obtain a good permutation: this will happen with high probability as one at least, $\pi^*$, exists. Secondly, we show that, under the conditions of Theorem \ref{th:ub}, any good permutation returned is guaranteed to overlap with $\pi^*$ on a fraction $\alpha$ of its nodes (Proposition \ref{prop:overlap}). These two propositions together imply the conclusions of Theorem \ref{th:ub}.

\begin{definition} \label{def:good.perm}
	Let $0<\alpha<1$. A permutation $\pi$ is said to be \emph{good} if 
	$$\sum_{i=1}^n \textbf{1}_{ \left\{\sum_{j \neq i} A_{ij}B_{\pi(i)\pi(j)}\geq\frac{nqs}{2}\right\} }\geq n \cdot \frac{1+\alpha}{2},$$
	where $\textbf{1}$ denotes the indicator function.
\end{definition}
In other words, if we look at the intersection graph of $G_A$ and $G_B$ to which $\pi$ has been applied, there are at least $n\cdot \frac{1+\alpha}{2}$ nodes that have degree greater than or equal to $\frac{nqs}{2}.$

\begin{proposition}\label{lem:alg.term}
	For any $(A,B)$ generated as described in Section \ref{sec:math.model} and under the conditions given in Theorem \ref{th:ub}, $\pi^*$ is whp a good permutation.
\end{proposition}
The proof of this proposition is given in the Electronic Companion to this paper. The key idea is that the intersection graph of $G_A$ and $G_B'$ is Erd\H{o}s-R\'enyi of parameters $(n,qs)$. Hence, one can use the $k$-core results given in Lemma \ref{lem:k,core}. In particular, the proof in the Electronic Companion involves showing that for any $0<\alpha<1$, and under the conditions of Theorem \ref{th:ub}, we have that $nqs>c_{nqs/2}$ and that $\psi_{nqs/2}(\mu_{nqs/2}(nqs)) \geq \frac{1+\alpha}{2}$. From Lemma \ref{lem:k,core}, this implies that there exists with high probability in this intersection graph an $nqs/2$-core of size greater than or equal to $n \cdot \frac{1+\alpha}{2}$, which in turn implies that $\pi^*$ is good. Proposition \ref{lem:alg.term} serves to show that the algorithm will terminate and output a good permutation with high probability.

We now move onto Proposition \ref{prop:overlap}. Note that the proofs of some technical lemmas used to show Proposition \ref{prop:overlap} have been differed to the Electronic Companion for clarity. We also make use of the following result that appears in \cite{cullina2018partial}, the proof of which we also include in the Electronic Companion (Section \ref{sec:cullina}) for completeness.

\begin{lemma}[Lemma 3 \cite{cullina2018partial}]\label{lem:cullina}
	Let $\tau \geq 0$, $q_1\geq 0$, and $q_2\geq 0$. Then 
	\begin{align}\label{eq:cullina.1}
	\min_{z\geq 0} z^{-\tau} e^{q_2(z^2-1)+q_1(z-1)} \leq \zeta^\tau
	\end{align}
	where $\zeta\mathrel{\mathop{:}}=\max \{\sqrt{2}e\frac{q_1}{\tau},4e \left(\frac{q_2}{\tau}\right)^{1/2}\}$. Furthermore, if $z^*$ is the minimizer of this expression,
	\begin{align} \label{eq:ub.z*}
	z^{*2}=\frac{\tau-q_1z^*}{2q_2}\leq \frac{\tau}{2q_2}.
	\end{align}
	\end{lemma}

\begin{proposition}\label{prop:overlap}
	Let $0<\alpha<1$, $\beta>0$, $\gamma>0$. If the conditions of Theorem \ref{th:ub} hold, then a good permutation $\pi$ agrees whp with the ground truth $\pi^*$ on at least a fraction $\alpha$ of the nodes in the graph. In other words, 
	\begin{align} \label{eq:wts}
	\p(\forall \pi, [\pi \text{ good } \Rightarrow \beta(\pi,\pi^*)>\alpha]) \rightarrow 1 \text{ when } n \rightarrow \infty.
	\end{align}
\end{proposition}

\begin{proof}{\gh{Proof.}} We show that when conditions (\ref{eq:conditions.ub}) hold, we have $\lim_{n \rightarrow \infty}\p(\exists \pi: \pi \text{ good } \cap \beta(\pi,\pi^*) \leq \alpha) \rightarrow 0$, which implies (\ref{eq:wts}). Note that
\begin{align}\label{eq:existence.pi}
\p(\exists \pi: \pi \text{ good } \cap \beta(\pi,\pi^*)\leq \alpha)=\p(\cup_{\{\pi: \beta(\pi,\pi^*)\leq \alpha\}} \{\pi : \pi \text{ good}\})\leq \sum_{\{\pi:\beta(\pi,\pi^*) \leq \alpha\}} \p(\pi \text{ good}).
\end{align}
	
Let $\pi$ be such that $\beta(\pi,\pi^*)=\epsilon$ where $\epsilon\leq \alpha$. Our goal is to upper bound $\p(\pi \text{ good})$. To do this, we make use of the following fact: if $\pi$ is good, then among all $n(1-\epsilon)$ nodes that $\pi$ and $\pi^*$ do not overlap on, at least $\frac{n(1+\alpha)}{2}-n\epsilon=\frac{n(1+\alpha-2\epsilon)}{2} \geq \frac{n(1-\alpha)}{2}$ nodes must have degree greater than or equal to $\frac{nqs}{2}$, i.e., $\sum_{j \neq i} A_{ij} B_{\pi(i)\pi(j)} \geq \frac{nqs}{2}$ for a set of indices $i$ in $\{i~|~\pi(i)\neq \pi^*(i)\}$ of cardinality greater than or equal to $n \cdot \frac{1-\alpha}{2}$. This further implies that 
$$\sum_{\{i~|~ \pi(i)\neq \pi^*(i)\}} \sum_{j\neq i} A_{ij}B_{\pi(i)\pi(j)} =\sum_{(i,j)\in S_2}A_{ij}B'_{p(i)p(j)} \geq \frac{nqs \cdot n(1-\alpha)/2}{2},$$ where $p=\pi\circ \pi^{*-1}$ and $S_2$ is the set given in Section \ref{subsec:perms}. This implication enables us to upper bound $\p(\pi \text{ good})$ in the following fashion:
$$\p(\pi \text{ good}) \leq \p\left( \sum_{(i,j)\in S_2}A_{ij}B'_{p(i)p(j)} \geq \frac{nqs \cdot n(1-\alpha)/2}{2} \right).$$
Recalling that $S_2=S_2^1 \cup S_2^2$, we have that $$\p(\pi\text{ good}) \leq \p \left(\sum_{(i,j)\in S_2^1}A_{ij}B'_{p(i)p(j)}+\sum_{(i,j)\in S_2^2}A_{ij}B'_{p(i)p(j)}\geq \frac{nqs \cdot n(1-\alpha)}{4} \right).$$
As it turns out, the variables $\{A_{ij}B'_{p(i)p(j)}\}_{(i,j) \in S_2^2}$ are independent from $\{A_{ij}B'_{p(i)p(j)}\}_{(i,j) \in S_2^1}$. This gives rise to the following classical Chernoff inequality:
\begin{align}\label{eq:chernoff}
\p(\pi \text{ good})\leq \inf_{t>0} e^{-tu}  \E\left[e^{t\sum_{(i,j)\in S_2^1} A_{ij}B'_{p(i)p(j)}}\right] \cdot \E\left[e^{t\sum_{(i,j)\in S_2^2} A_{ij}B'_{p(i)p(j)}}\right],
\end{align}
where $u\mathrel{\mathop{:}}=\frac{nqs\cdot n(1-\alpha)}{4}$. It is easy to show that for large enough $n$, $u-\E[\sum_{(i,j)\in S_2}A_{ij}B'_{p(i)p(j)}]>0$. Indeed, using (\ref{eq:size.S21.S22}) and the assumptions $s > \frac{8}{1-\alpha} \cdot q$ and $n(1-\epsilon)qs=o(nqs\cdot n(1-\alpha)/4)$ from Theorem \ref{th:ub}, we have
\begin{align*}
u-\E[\sum_{(i,j)\in S_2}A_{ij}B'_{p(i)p(j)}]&=\frac{nqs\cdot n(1-\alpha)}{4}-\sum_{(i,j)\in S_2^1}\E[A_{ij}B'_{p(i)p(j)}]-\sum_{(i,j)\in S_2^2}\E[A_ijB'_{p(i)p(j)}]\\
&\geq \frac{nqs\cdot n(1-\alpha)}{4}-n(1-\epsilon)qs-n(1-\epsilon)(n-1)q^2>0
\end{align*}
for large enough $n$. It remains to compute the two moment generating functions that appear in (\ref{eq:chernoff}). It is quite straightforward to see that, for $t>0$,
\begin{align*}
\E\left[e^{t\sum_{(i,j)\in S_2^1} A_{ij}B'_{p(i)p(j)}}\right]=\E\left[e^{t\sum_{(i,j)\in S_2^1} A_{ij}B'_{ji}}\right]\leq \E\left[e^{t(A_{ij}B'_{ji}+A_{ji}B'_{ij})}\right]^\frac{|S_2^1|}{2}&\leq \E\left[e^{2t(A_{ij}B'_{ji}+A_{ji}B'_{ij})}\right]^\frac{|S_2^1|}{4}\\
&=(\pii e^{4t}+1-\pii)^{\frac{|S_2^1|}{4}}\\
&=e^{\frac{|S_2^1|}{4} \cdot \log(1+\pii (e^{4t}-1))}.
\end{align*}

\gh{We now focus on $\E\left[e^{t\sum_{(i,j)\in S_2^2} A_{ij}B'_{p(i)p(j)}}\right]$: this expression is not as straightforward to analyze as the variables involved all depend on one another in a more complicated way. A first step in simplifying this analysis is to use $\bar{S}_2^2$ rather than $S_2^2$ as it enables us to use the concept of \emph{cycles} introduced in Section \ref{subsec:perms}. As the matrices $A$ and $B'$ are symmetric and nonnegative, and $t>0$, and in light of the definition of $\bar{S}_2^2$, we have:
$$t\sum_{(i,j)\in S_2^2} A_{ij}B'_{p(i)p(j)} \leq 2t\sum_{\{i,j\}\in \bar{S}_2^2} A_{ij}B'_{p(i)p(j)}.$$
Together with (\ref{eq:partition.S22}), this implies
$$\E\left[e^{t\sum_{(i,j)\in S_2^2} A_{ij}B'_{p(i)p(j)}}\right] \leq \E\left[e^{2t\sum_{\{i,j\}\in \bar{S}_2^2} A_{ij}B'_{p(i)p(j)}}\right] \leq \E\left[e^{2t \sum_{k=2}^{|\bar{S}_2^2|}\sum_{C \in \mathcal{C}_k}\sum_{\{i,j\}\in C} A_{ij}B'_{p(i)p(j)}}\right].$$
Given $\{i,j\} \in \bar{S}_2^2$, $A_{\{i,j\}}=A_{ij}=A_{ji}$ and $B'_{\{i,j\}}=B'_{ij}=B'_{ji}$ will appear exactly once and their appearance will be in the \emph{same} cycle. As $A_{\{i,j\}}$ and $B'_{\{i,j\}}$ are independent of all other random variables, it follows that the cycles are all independent of one another. 
Thus, we obtain:
\begin{align}\label{eq:ub.s22}
\E\left[e^{t\sum_{(i,j)\in S_2^2} A_{ij}B'_{p(i)p(j)}}\right] \leq \prod_{k=2}^{|\bar{S}_2^2|} \prod_{C \in \mathcal{C}_k} \E\left[e^{2t\sum_{\{i,j\}\in C} A_{ij}B'_{p(i)p(j)}}\right].
\end{align}
The difficulty then becomes computing the moment generating function 
$$\E\left[e^{2t\sum_{\{i,j\}\in C} A_{ij}B'_{p(i)p(j)}}\right].$$
As it turns out, this can be done by noting the special structure of the sum: the pair $(A_{\{i,j\}}, B'_{\{i,j\}})$ only appears once, at the extremity of the sum, similarly for $(A_{\{p(i),p(j)\}}, B'_{\{p(i),p(j)\}})$, though they are once removed from the extremity of the sum, and so on and so forth. By conditioning recursively on these pairs, one can obtain the moment generating function of interest as a function of four sequences defined via recursive relationships. Solving each of the four recursions then gives us the moment generating function, which only depends on the length of the cycle $C$. We omit the details of this proof in the main text: it can be found in Section \ref{sec:proofs.ec} of the Electronic Companion. We give here only the end result, namely that for a cycle $C$ of length $k$, 
$$\E\left[e^{2t\sum_{\{i,j\}\in C} A_{ij}B'_{p(i)p(j)}}\right]=\left( \frac{T+\sqrt{T^2-4D}}{2}\right)^{k}+\left(\frac{T-\sqrt{T^2-4D}}{2} \right)^{k}$$
where $T=\pii(e^{2t}-1)+1 \text{ and } D=\sigma^2(e^{2t}-1).$ 
Using this result in (\ref{eq:ub.s22}), we obtain:
\begin{align*}
\E\left[e^{t\sum_{(i,j)\in S_2^2} A_{ij}B'_{p(i)p(j)}}\right] \leq \prod_{k=2}^{|\bar{S}_2^2|} \left( \left( \frac{T+\sqrt{T^2-4D}}{2}\right)^{k}+\left(\frac{T-\sqrt{T^2-4D}}{2} \right)^{k} \right)^{\ell_k}.
\end{align*}
We then leverage Lemma \ref{prop:arithm} from the Electric Companion to further upper bound this quantity:
\begin{align*}
\E\left[e^{t\sum_{(i,j)\in S_2^2} A_{ij}B'_{p(i)p(j)}}\right] &\leq \prod_{k=2}^{|\bar{S}_2^2|} \left( \left( \frac{T+\sqrt{T^2-4D}}{2}\right)^{2}+\left(\frac{T-\sqrt{T^2-4D}}{2} \right)^{2} \right)^{k \cdot \ell_k/2}\\
&= \left( \left( \frac{T+\sqrt{T^2-4D}}{2}\right)^{2}+\left(\frac{T-\sqrt{T^2-4D}}{2} \right)^{2} \right)^{\sum_{k=2}^{|\bar{S}_2^2|}k \cdot \ell_k/2}.
\end{align*}
Recalling (\ref{eq:partition.S22}) and as 
$$\left( \frac{T+\sqrt{T^2-4D}}{2}\right)^{2}+\left(\frac{T-\sqrt{T^2-4D}}{2} \right)^{2}=T^2-2D,$$
it follows that 
\begin{align*}
\E\left[e^{t\sum_{(i,j)\in S_2^2} A_{ij}B'_{p(i)p(j)}}\right] &\leq \left(T^2-2D\right)^{|\bar{S}_2^2|/2}\\ &=\left(\pii^2(e^{2t}-1)^2+2\qii(e^{2t}-1)+1\right)^{|\bar{S}_2^2|/2}\\
&=\left( \pii^2(e^{4t}-1)+2(\qii-\pii^2)(e^{2t}-1)+1\right)^{|\bar{S}_2^2|/2}\\
&=e^{\frac{|\bar{S}_2|^2}{2} \cdot \log \left(1+\pii^2(e^{4t}-1)+2(\qii-\pii^2)(e^{2t}-1)\right)}.
\end{align*}
Putting everything together in (\ref{eq:chernoff}), we get:}
\begin{align*}
\p(\pi \text{ good}) &\leq \inf_{t>0} e^{-t\frac{n\pii \cdot n(1-\alpha)}{4}} \cdot e^{\frac{|S_2^1|}{4} \cdot \log(1+\pii (e^{4t}-1))} \cdot e^{\frac{|\bar{S}_2|^2}{2} \cdot \log \left(1+\pii^2(e^{4t}-1)+2(\qii-\pii^2)(e^{2t}-1)\right)} \nonumber \\
&\leq \inf_{t>0} e^{-t\frac{n\pii \cdot n(1-\alpha)}{4}+ \frac{n}{4} \cdot \log(1+\pii (e^{4t}-1)) + \frac{n^2}{4} \cdot \log \left(1+\pii^2(e^{4t}-1)+2(\qii-\pii^2)(e^{2t}-1)\right)}\\
&\leq \inf_{t>0} e^{-t\frac{n\pii \cdot n(1-\alpha)}{4}} e^{\frac{n}{4} \pii(e^{4t}-1)} \cdot  e^{\frac{n^2}{4} \cdot  \left(\pii^2(e^{4t}-1)+2(\qii-\pii^2)(e^{2t}-1)\right)}\\
&\leq \left(\inf_{t>0} e^{-t \cdot (1-\alpha)} e^{\frac{e^{4t}-1}{n}} \cdot e^{\pii(e^{4t}-1)+2\frac{(\qii-\pii^2)}{\pii}(e^{2t}-1)} \right)^{\frac{n \cdot n \pii}{4}}\\
&\underset{z=e^{2t}}{=} \left( \inf_{z>0} z^{-(1-\alpha)/2} e^{\frac{z^2-1}{n}} e^{\pii(z^2-1)+2\frac{\qii-\pii^2}{\pii}(z-1)}\right)^{\frac{n \cdot n \pii}{4}},
\end{align*}
where we have used (\ref{eq:size.S21.S22}) in the second inequality to obtain $|S_1^1|\leq n(1-\epsilon)\leq n$ and $|\bar{S}_2^2|\leq \frac{n^2}{2}$, and the fact that $\log(1+x)\leq x$ for $x\geq 0$ for the third inequality. We now leverage Lemma~\ref{lem:cullina} to conclude. In our case,
\begin{align} \label{eq:def.qts.cullina}
\tau=\frac{1-\alpha}{2},~q_1=2\frac{\qii-\pii^2}{\pii},~q_2=\pii.
\end{align} 
Let $z^*$ be as defined in Lemma \ref{lem:cullina}: from (\ref{eq:ub.z*}), we have
$z^{*2} \leq \frac{\tau -q_1 z^{*}}{2q_2} \leq \frac{\tau}{2q_2}$ and so $$\frac{z^{*2}-1}{n} \leq \frac{z^{*2}}{n}  \leq \frac{1-\alpha}{4n\pii}.$$
Leveraging this to upper bound $e^{\frac{z^2-1}{n}}$ independently of $z$ and then using (\ref{eq:cullina.1}) with the definitions in (\ref{eq:def.qts.cullina}) in mind, it follows that
$$\p(\pi \text{ good}) \leq e^{\frac{n(1-\alpha)}{16}} \left( \frac{1}{\zeta}\right)^{-\frac{n(1-\alpha) n\pii}{8}},$$
where $\zeta$ is defined as in Lemma \ref{lem:cullina}, i.e., $$\frac{1}{\zeta}=\min \left\{ \frac{\tau}{\sqrt{2}eq_1},\frac{\tau^{1/2}}{4eq_2^{1/2}}\right\}.$$ We show that $\log \left( \frac{1}{\zeta}\right) \geq \min\{\gamma,\beta\}\log n$ by lower bounding each term under the assumptions of Theorem \ref{th:ub}. First,
$$\log \left(\frac{\tau}{\sqrt{2}eq_1} \right)=\log\left( \frac{\pii \cdot(1-\alpha)}{4\sqrt{2} \cdot e (\qii-\pii^2)} \right) \geq \beta \log(n)+O(1).$$
Secondly, noting that for the assumptions in (\ref{eq:conditions.ub}) to make sense, it must be the case that $\gamma <\frac{1}{2}$ and so $n^{1/2} \geq n^{\gamma}$:
$$\log \left( \frac{\tau^{1/2}}{4eq_2^{1/2}} \right)=\log \left( \frac{\sqrt{1-\alpha}}{4e\sqrt{2}\sqrt{qs}} \right) \geq \gamma \log (n)+O(1).$$
As $|\{\pi~|~\beta(\pi,\pi^*)\leq \alpha\}|\leq |\{\pi \in P_n\}|=n!\leq n^{\log n}$, plugging everything back into (\ref{eq:existence.pi}) gives us
$$\p(\exists \pi: \pi \text{ good } \cap \beta(\pi,\pi^*)< \alpha) \leq e^{n \log n+\frac{n(1-\alpha)}{16}-\frac{n(1-\alpha)n\pii (\min\{\gamma,\beta\} \log n+O(1))}{8}}\leq e^{-n\log n +O(n)},$$
where the final inequality is implied by the assumptions in Theorem \ref{th:ub}. As the right hand side of this inequality converges to $0$ when $n\rightarrow \infty$, we get our result.
\end{proof}

\newpage
\appendix

\section{Proof of Proposition \ref{lem:alg.term}}

The proof of Proposition \ref{lem:alg.term} makes use of the following lemmas.

\begin{lemma}[Berry-Esseen Theorem; see, e.g., \cite{korolev2010upper}]\label{lem:barry.esseen}
	Let $X_1,X_2,\ldots $ be independent and identically distributed random variables such that $\E[X_1]=0$, $\E[X_1^2]=\sigma^2>0$ and $\E[|X_1|^3]<\beta^3<\infty$. For $n\geq 1$, define
	$$S_n=\frac{X_1+\ldots+X_n}{\sigma\sqrt{n}}$$
	and let $F_n(x)=\p[S_n\leq x]$ be its cdf. Let $\phi(x)$ be the cdf of a standard normal distribution. For any $x$ and any $n\geq 1$, it holds
	$$|F_n(x)-\phi(x)|\leq \frac{0.34445\beta^3+0.16844}{\sqrt{n}}.$$
\end{lemma}

\begin{lemma}\label{lem:lb.psi}
Let $\mu>0$ and recall the notation $\psi_j(\mu)=\p(Po(\mu) \geq j)$, where $Po(\mu)$ is a Poisson random variable of parameter $\mu$ and $j$ is a nonnegative integer. Denote by $\phi$ the cdf of a standard normal. We have
$$\psi_j(\mu) \geq 1-\phi \left(\frac{j-\mu}{\sqrt{\mu}} \right)-\frac{0.55}{\sqrt{\lceil\mu\rceil}}.$$
In particular, for $\lambda \geq 20$,
$$\psi_{\lambda/2-1}(2\lambda/3) \geq 0.7$$
	\end{lemma}

\begin{proof}{\gh{Proof.}}
We can write $Po(\mu)=\sum_{i=1}^{\lceil \mu \rceil} P_i$ where $P_i, i=1,\ldots,\mu$ are independent Poisson random variables with parameter $\mu/\lceil \mu \rceil\leq 1$. Let $X_i=P_i-\frac{\mu}{\lceil \mu \rceil}$: we have that $\E[X_i]=0$, $\E[X_i^2]=\mu/\lceil \mu \rceil$, and $\E[X_i^3]=\mu/\lceil \mu \rceil.$
From this, we have that 
\begin{align*}
\psi_{j}(\mu)=\p(Po(\mu)\geq j)&=\p\left(\sum_{i=1}^{\lceil \mu \rceil} (Pi-\E[P_i])\geq j-\mu\right)\\
&=\p\left(\frac{\sum_{i=1}^{\lceil \mu \rceil} X_i}{\sqrt{\mu/\lceil \mu \rceil} \sqrt{\lceil \mu \rceil}}\geq \frac{j-\mu}{\sqrt{\mu/\lceil \mu \rceil} \sqrt{\lceil \mu \rceil}} \right)\\
&=\p\left(\frac{\sum_{i=1}^{\lceil \mu \rceil} X_i}{\sqrt{\mu}}\geq \frac{j-\mu}{\sqrt{\mu}} \right)\\
&=1-F_{\lceil \mu \rceil} \left( \frac{j-\mu}{\sqrt{\mu}}\right).
\end{align*}
Using Lemma \ref{lem:barry.esseen} and the fact that $\E[|X_i|^3]=\frac{\mu}{\lceil \mu \rceil} \leq 1$, we have that for any $x$,
$$||1-F_{\lceil \mu \rceil}(x)|-|1-\phi(x)||\leq |F_{\lceil \mu \rceil}(x)-\phi(x)|\leq \frac{0.55}{\sqrt{\lceil \mu\rceil}},$$
which implies $1-F_{\lceil \mu \rceil}(x)\geq 1-\phi(x)-\frac{0.55}{\sqrt{\lceil \mu \rceil}}$ and so,
$$\psi_j(\mu) \geq 1-\phi \left(\frac{j-\mu}{\sqrt{\mu}} \right)-\frac{0.55}{\sqrt{\lceil \mu \rceil}}.$$
Now, assume that $\lambda \geq 20$. We have that
$$\psi_{\lambda/2-1}(2\lambda/3) \geq 1-\phi \left(\frac{\lambda/2-1-2\lambda/3}{\sqrt{2\lambda/3}} \right)-\frac{0.55}{\sqrt{\lceil 2\lambda/3 \rceil}}=1-\phi \left(\frac{-\lambda/6-1}{\sqrt{2\lambda/3}} \right)-\frac{0.55}{\lceil \sqrt{2\lambda/3} \rceil}.$$
When $\lambda \geq 20$, we have that $\lambda\mapsto-\frac{\lambda/6+1}{\sqrt{2\lambda/3}}$ is decreasing and that $\lambda \mapsto -\frac{0.55}{\lceil\sqrt{2\lambda/3} \rceil}$ is increasing. From this, we get
$$ \psi_{\lambda/2-1}(2\lambda/3) \geq 1-\phi \left(\frac{-20/6-1}{\sqrt{2\cdot 20/3}} \right)-\frac{0.55}{\sqrt{\lceil 2\cdot 20/3 \rceil}}\geq 0.7.$$
\end{proof}

\begin{proof}{Proof of Proposition \ref{lem:alg.term}.}
The permutation $\pi^*$ is good if the intersection graph of $G_A$ and $G_{B'}$ has at least $n\cdot \frac{1+\alpha}{2}$ nodes that have degree greater than or equal to $\frac{nqs}{2}$. By definition of a $k$-core, if this intersection graph has an $nqs/2$-core of size greater than or equal to $n(1+\alpha)/2$ whp, then $\pi^*$ is good whp. 

The intersection graph of $G_A$ and $G_{B'}$ is a graph over $n$ nodes such that there is an edge between nodes $i$ and $j$ if and only if there is an edge in $G_A$ and $G_{B'}$ between $i$ and $j$. As $\p((A_{ij},B'_{ij})=(1,1))=qs$ and $(A_{ij}, B'_{ij})$ is independent of $(A_{kl}, B'_{kl})$ for any $(k,l) \neq (i,j)$, it follows that the intersection graph of $G_A$ and $G_{B'}$ is an Erd\H{o}s-R\'enyi graph of parameters $(n,qs)$. If we are able to show that under the conditions of Theorem \ref{th:ub},
\begin{enumerate}[(i)]
	\item $nqs > c_{nqs/2}$
	\item $\psi_{nqs/2}(\mu_{nqs/2}(nqs)) \geq n \cdot \frac{1+\alpha}{2}$
\end{enumerate}
then from Lemma \ref{lem:k,core}, the $nqs/2$-core of the intersection graph of $G_A$ and $G_{B'}$ contains a fraction greater than or equal to $\frac{1+\alpha}{2}$ of the nodes in the graph whp and Proposition \ref{lem:alg.term} follows.

We start by proving (i). Recall that $c_{nqs/2} =\inf_{\mu>0}\frac{\mu}{\psi_{nqs/2-1}(\mu)}.$ Hence, for any $\mu>0$, in particular for $\mu=2nqs/3>0$,
$$c_{nqs/2} \leq \frac{2nqs/3}{\psi_{nqs/2-1}(2nqs/3)} \leq \frac{2nqs/3}{0.7}<nqs,$$
where we have used Lemma \ref{lem:lb.psi} for the second inequality as the conditions of Theorem \ref{th:ub} imply that $\lambda \geq 20$.

We now prove (ii) by showing that under the conditions of Theorem \ref{th:ub}, $\frac{2nqs}{3} \leq \mu_{nqs/2}(nqs) \leq nqs$ where $\mu_{nqs/2}(nqs)$ is the largest solution to $\frac{\mu}{\psi_{nqs/2-1}(\mu)}=nqs$. This implies that
\begin{align*}
\psi_{nqs/2}(\mu_{nqs/2}(nqs))&=\p(Po(\mu_{nqs/2}(nqs)) \geq nqs/2)\\
&=1-\p(Po(\mu_{nqs/2}(nqs))) \leq \mu_{nqs/2}(nqs)-(\mu_{nqs/2}(nqs)-nqs/2))\\
&\geq 1-\p(Po(\mu_{nqs/2}(nqs)) \leq \mu_{nqs/2}(nqs)-nqs/6)\\
&\geq 1-e^{-(nqs/6)^2/2(\mu_{nqs/2}(nqs)+nqs/6)}\\
&\geq 1-e^{-(nqs)^2/(36\cdot 7/3 \cdot nqs)}=1-e^{-nqs/84},
\end{align*}
where we have used a standard Chernoff bound for Poisson random variables for the second inequality, see, e.g. \cite{boucheron2013concentration}. As the conditions of Theorem \ref{th:ub} require that $nqs \geq 84 \cdot \log \left( \frac{2}{1-\alpha}\right)$, we get that $\psi_{nqs/2}(\mu_{nqs/2}(nqs)) \geq \frac{1+\alpha}{2}$.

To show that $\frac{2nqs}{3} \leq \mu_{nqs/2}(nqs) \leq nqs$, let $f_{nqs}(\mu)=\mu-nqs \cdot \psi_{nqs/2-1}(\mu)$, which is a continuous function of $\mu$ for fixed $nqs$. Here, $\mu_{nqs/2}(nqs)$ is the largest solution to $f_{nqs}(\mu)=0$. It is straightforward to see that for $\mu>nqs$, $f_{nqs}(\mu)>0$ as $\psi_{nqs/2-1} \leq 1$ for any value that $nqs$ takes. Furthermore, from Lemma \ref{lem:lb.psi},
$$f_{nqs}(2nqs/3)=\frac{2nqs}{3}-nqs\cdot \psi_{nqs/2-1}(2nqs/3) \leq \frac{2nqs}{3}-nqs\cdot 0.7<0.$$ Combining these results, we get that $\mu_{nqs/2}(nqs)$ must be less than or equal to $nqs$ and that there is a solution $f_{nqs}(\mu)=0$ in the interval $[2nqs/3,nqs]$. Our result immediately follows. 
\end{proof}

\section{Proofs of auxiliary results for Proposition \ref{prop:overlap}}\label{sec:proofs.ec}
	\begin{lemma}\label{prop:arithm}
		Let $n \in \mathbb{N}$ and let $ 2 \leq k \leq n$. Furthermore, let $\alpha,\beta$ be positive real numbers. We have
		$$(\alpha^k+\beta^k)^{n/k} \leq (\alpha^2+\beta^2)^{n/2}.$$
		
	\end{lemma}
	
	\begin{proof}{\gh{Proof.}}
		We have
		\begin{align*}
		\frac{(\alpha^2+\beta^2)^{n/2}}{(\alpha^k+\beta^k)^{n/k}}=\frac{(\alpha^2+\beta^2)^{k/2 \cdot n/k}}{(\alpha^k+\beta^k)^{n/k}}= \left(\frac{(\alpha^2+\beta^2)^{k/2}}{(\alpha^k+\beta^k)} \right)^{n/k}.
		\end{align*}
		As $n/k>1$, we simply need to check whether $(a^2+\beta^2)^{k/2} \geq (\alpha^k+\beta^k)$. If $k$ is divisible by $2$, this is straightforward from the Binomial theorem as $\alpha>0$ and $\beta>0$. If $k$ is not divisible by $2$, then $k-1$ is divisible by $2$ and $(\alpha^2+\beta^2)^{k}=(\alpha^2+\beta^2)^{(k-1)/2} \cdot \sqrt{\alpha^2+\beta^2}$. By applying the Binomial theorem to the first part of the sum and noting that $\sqrt{\alpha^2+\beta^2} \geq \alpha$ and $\beta$, the result follows.
	\end{proof}
\begin{lemma} \label{th:sequence}
	Let $(A_0,B_0)$, $(A_1,B_1), \ldots, (A_n,B_n)$ be $n+1$ pairs of random variables taking values in $\{0,1\}^2$. We suppose that $(A_i,B_i)$ are distributed following the probability distribution $P$ given in Section \ref{subsec:major.results.lb}. In particular, the covariance between $A_i$ and $B_i$ is given by $\sigma^2>0$. We further assume that all pairs $\{(A_i,B_i)\}$ are independent of one another. Define
	$$W\mathrel{\mathop{:}}=A_0B_1+A_1B_2+\ldots+A_{n-1}B_n+A_nB_0.$$
	For any $t>0$, we have that
	$$\E[e^{tW}]=\left( \frac{T+\sqrt{T^2-4D}}{2}\right)^{n+1}+\left(\frac{T-\sqrt{T^2-4D}}{2}\right)^{n+1}$$
	where $T=\pii(e^{t}-1)+1$ and $D=\sigma^2(e^t-1)$. 
\end{lemma}

\begin{remark}
	Note that if $\sigma^2=0$, i.e., $A_{i}$ and $B_{i}$ are uncorrelated, this moment generating function simplifies to that of a Binomial random variable of parameters $(n,qs)$, as it should.
\end{remark}

Our proof will make use of the following result from \cite{williams1992n}.
 	
\begin{lemma}\label{lem:Williams}
		Let $M$ be a $2\times 2$ matrix with eigenvalues $\alpha$ and $\beta$. Then, if $\alpha \neq \beta$, for all $n\geq 1$,
		$$M^n=\alpha^n \left(\frac{M-\beta I}{\alpha-\beta}\right)+\beta^n \left( \frac{M-\alpha I}{\beta-\alpha} \right).$$
\end{lemma}

\begin{proof}{Proof of Theorem \ref{th:sequence}.}
Let
\begin{align*}
\uoo_{n}&=\E[e^{t(A_1B_2+A_2B_3+\ldots+A_{n-2}B_{n-1}+A_{n-1}B_n)}],\\
\uoi_{n}&=\E[e^{t(A_1B_2+A_2B_3+\ldots+A_{n-2}B_{n-1}+A_{n-1}B_n+A_n)}],\\
\uio_{n}&=\E[e^{t(B_1+A_1B_2+A_2B_3+\ldots+A_{n-2}B_{n-1}+\ldots+A_{n-1}B_n)}],\\
\uii_{n}&=\E[e^{t(B_1+A_1B_2+A_2B_3+\ldots+A_{n-2}B_{n-1}+\ldots+A_{n-1}B_n+A_n)}].
\end{align*}
Note that by conditioning on $(A_0,B_0)$, we have that $\E[e^{tW}]=\poo \uoo_n+\poi \uoi_n+\pio \uio_n+\pii \uii_n$. We thus study these four sequences. By conditioning on $(A_1, B_1)$, it is easy to see that the following formulas hold true:
	\begin{align*}
	\uoo_n&=(\poo+\poi)\uoo_{n-1}+(\pio+\pii)\uio_{n-1}\\
	\uoi_n&=(\poo+\poi)\uoi_{n-1}+(\pio+\pii)\uii_{n-1}\\
	\uio_n &=(\poo+\poi e^t)\uoo_{n-1}+(\pio+\pii e^t)\uio_{n-1}\\
	\uii_n&=(\poo+\poi e^t)\uoi_{n-1}+(\pio+\pii e^t)\uii_{n-1}.
	\end{align*}
	Successive conditioning on $(A_{k},B_k)$ for $k\geq 2$ leads to these formulas holding for all $n\geq 2$. Hence, if we let 
	$$M\mathrel{\mathop{:}}=\begin{bmatrix} \poo+\poi ~ & ~\pio+\pii \\ \poo+\poi e^t~ & ~\pio+\pii e^t \end{bmatrix}$$
	then 
	$$\begin{bmatrix} \uoo_n \\ \uio_{n} \end{bmatrix} =M \cdot \begin{bmatrix} \uoo_{n-1} \\ \uio_{n-1} \end{bmatrix}~ \text{ and } ~\begin{bmatrix} \uoi_n \\ \uii_{n} \end{bmatrix} =M \cdot \begin{bmatrix} \uoi_{n-1} \\ \uii_{n-1} \end{bmatrix}.$$
We further have the following initial conditions
	\begin{align*}
	\uoo_1&=\E[e^{t(0\cdot B_n+A_n \cdot 0)}]=1\\
	\uoi_1&=\E[e^{t(0\cdot B_n+A_{n})}]=1+(\poi+\pii)(e^t-1)\\
	\uio_1&=\E[e^{t(B_{n}+A_n\cdot 0)}]=1+(\poi+\pii)(e^t-1)\\
	\uii_1&=\E[e^{-t(B_n+A_n)}]=1+2(\pii+\poi)(e^t-1)+\pii (e^{t}-1)^2
	\end{align*}
	Using these conditions, we get that 
	$$\begin{bmatrix} \uoo_n \\ \uio_{n} \end{bmatrix} =M^{n-1} \cdot \begin{bmatrix} \uoo_{1} \\ \uio_{1} \end{bmatrix}~ \text{ and } ~\begin{bmatrix} \uoi_n \\ \uii_{n} \end{bmatrix} =M^{n-1} \cdot \begin{bmatrix} \uoi_{1} \\ \uii_{1} \end{bmatrix}.$$
	The difficulty now is in computing $M^{n-1}$, which we will do via Lemma \ref{lem:Williams}. To this effect, we denote by $T$ the trace of $M$ and by $D$ its determinant. One can easily compute their values:
		$$T=\mbox{tr}(M)=\pii(e^t-1)+1 \text{ and } D=\det(M)=\sigma^2(e^t-1).$$ Recall that the eigenvalues of a $2 \times 2$ matrix satisfy the equation $x^2-Tx+D=0$. The discriminant here is $T^2-4D$, which is positive:
	\begin{align*}
	\Delta&=T^2-4D=\pii^2(e^{t}-1)^2+2(\pii-2\sigma^2)(e^t-1)+1=\left(\pii(e^t-1)+\frac{\pii-2\sigma^2}{\pii} \right)^2+\frac{4\sigma^2(\pii-\sigma)}{\pii^2}>0
	\end{align*}
	Hence solutions to this equation exist. We write:
	$$\alpha=\frac{T+\sqrt{T^2-4D}}{2} \text{ and } \beta=\frac{T-\sqrt{T^2-4D}}{2}.$$
	In accordance with Lemma \ref{lem:Williams}, we now work on computing $M-\beta I$ and $M-\alpha I$. We have:
	\begin{align*}
	M-\beta I&=R+\frac{\sqrt{\Delta}}{2}I,
	\end{align*}
	where $$R\mathrel{\mathop{:}}=\begin{bmatrix} \frac{\poo-\pii}{2} -\frac{\pii}{2}(e^t-1)~&~\pio+\pii \\ \poo +\poi e^t & -(\frac{\poo-\pii}{2})+\frac{\pii}{2}(e^t-1) \end{bmatrix}.$$ We denote by $r_{ij}$ its entries. Note that $r_{11}=-r_{22}$. Likewise, we have 
	\begin{align*}
	M-\alpha I &=R-\frac{\sqrt{\Delta}}{2}I.
	\end{align*}
	It follows that 
	$$M^n=\frac{\alpha^n}{\sqrt{\Delta}} (R+\frac{\sqrt{\Delta}}{2}I) -\frac{\beta^n}{\sqrt{\Delta}} (R-\frac{\sqrt{\Delta}}{2}I)=\left(\frac{\alpha^n-\beta^n}{\sqrt{\Delta}}\right)R+\frac{\alpha^n+\beta^n}{2} \cdot I$$
	and from this, for any $n\geq 1$,
	\begin{align*}
	\uoo_n&=\left(\frac{\alpha^{n-1}-\beta^{n-1}}{\sqrt{\Delta}}\right) (r_{11}\uoo_1+r_{12}\uio_1)+\frac{\alpha^{n-1}+\beta^{n-1}}{2}\uoo_1\\
	\uio_n&=\left(\frac{\alpha^{n-1}-\beta^{n-1}}{\sqrt{\Delta}}\right) (r_{21}\uoo_1+r_{22}\uio_1)+\frac{\alpha^{n-1}+\beta^{n-1}}{2}\uio_1\\
	\uoi_n&=\left(\frac{\alpha^{n-1}-\beta^{n-1}}{\sqrt{\Delta}}\right)(r_{11}\uoi_1+r_{12}\uii_1)+\frac{\alpha^{n-1}+\beta^{n-1}}{2}\uoi_1\\
	\uii_n&=\left(\frac{\alpha^{n-1}-\beta^{n-1}}{\sqrt{\Delta}}\right)(r_{21}\uoi_1+r_{22}\uii_1)+\frac{\alpha^{n-1}+\beta^{n-1}}{2}\uii_1.
	\end{align*}
	This can now be used to compute $\E[e^{tW}]$. We have
	\begin{align*}
	\E[e^{tW}]&=\poo \uoo_{n}+\poi \uoi_{n}+\pio \uio_{n}+\pii \uii_{n}\\
	&=\left( r_{11}(\poo \uoo_1+\poi \uoi_1-\pio \uio_1 -\pii \uii_1) +r_{12}(\poo \uio_1+\poi \uii_1)+r_{21}(\pio \uoo_1 +\pii \uoi_1)\right)\\
	&\cdot \left(\frac{\alpha^{n-1}-\beta^{n-1}}{\sqrt{\Delta}}\right)+(\poo \uoo_1+\pio \uio_1+\poi \uoi_1+\pii \uii_1)\cdot \left(\frac{\alpha^{n-1}+\beta^{n-1}}{2}\right). 
	\end{align*}
	We compute each part separately. We have:
	\begin{align*}
	\poo \uoo_1+\pio \uio_1+\poi \uoi_1+\pii \uii_1 &=T^2-2D\\
	r_{11}(\poo \uoo_1+\poi \uoi_1-\pio \uio_1 -\pii \uii_1)&=\frac{1}{2}\left(\frac{\poo-\pii}{2} -\frac{\pii}{2}(e^t-1)\right)^2 \cdot T\\
	r_{12}(\poo \uio_1+\poi \uii_1)+r_{21}(\pio \uoo_1 +\pii \uoi_1) &=2T(\poi+\pii)(1-\pii-\poi+\poi(e^t-1))
	\end{align*}
	We combine the results above to obtain:
	\begin{align*}
	&r_{11}(\poo \uoo_1+\poi \uoi_1-\pio \uio_1 -\pii \uii_1) +r_{12}(\poo \uio_1+\poi \uii_1)+r_{21}(\pio \uoo_1 +\pii \uoi_1)=\frac{T^3}{2}-2DT
	\end{align*}
	which leads us to finally having
	\begin{align*}
	\E[e^{tW}]=\left(\frac{\alpha^{n-1}-\beta^{n-1}}{\sqrt{\Delta}}\right) \cdot \left( \frac{T^3}{2}-2DT\right)+\left(\frac{\alpha^{n-1}+\beta^{n-1}}{2}\right) \cdot (T^2-2D).
	\end{align*}
	As $\frac{T^3}{2}-2DT=\frac{T}{2} \cdot (T^2-4D)=\frac{T}{2} \cdot \Delta$, we can further simplify
	\begin{align*}
	\E[e^{tW}]&=\left(\alpha^{n-1}-\beta^{n-1}\right) \cdot \frac{T}{2}\cdot \sqrt{\Delta}+\left(\frac{\alpha^{n-1}+\beta^{n-1}}{2}\right) \cdot (T^2-2D)=\alpha^{n+1}+\beta^{n+1}
	\end{align*}
	\end{proof}

\section{Proof of Lemma \ref{lem:cullina}} \label{sec:cullina}

As mentioned previously, the proof of Lemma \ref{lem:cullina} is contained in \cite{cullina2018partial}. We repeat it here for completeness.

\begin{proof}{Proof of Lemma \ref{lem:cullina}.}
	The optimal choice of $z$ is obtained by differentiating $z \mapsto z^{-\tau} e^{q_2(z^2-1)+q_1(z-1)}$ and setting the derivative equal to 0, i.e., 
	$$(2q_2z+q_1)z^{-\tau} e^{q_2(z^2-1)+q_1(z-1)}-\tau z^{-\tau-1}e^{q_2(z^2-1)+q_1(z-1)}=0,$$
	which is equivalent to the quadratic equation
		\begin{align}\label{eq:quad.eq}
		2q_2z^2+q_1z-\tau=0.
		\end{align}
		This equation has one positive root and one negative root: the positive root is $z^*$. From this quadratic equation, we can obtain the expression $z^{*2}=\frac{\tau-q_1z^*}{2q_2}$. As $q_1>0$ and $z^*>0$, it follows that 
		$$z^{*2}\leq \frac{\tau}{2q_2},$$ which corresponds to (\ref{eq:ub.z*}). One can also give an explicit expression of $z^*$ by solving the quadratic equation:
		$$z^*=\frac{-q_1+\sqrt{q_1^2+8\tau q_2}}{4q_2}=\frac{2\tau}{q_1+\sqrt{q_1^2+8\tau q_2}}.$$
		Because $q_1 \leq \sqrt{q_1^2+8\tau q_2}$, we have the bounds 
		\begin{align} \label{eq:bounds.z*}
		\frac{\tau}{\sqrt{q_1^2+8\tau q_2}} \leq z^* \leq \frac{\tau}{q_1}.
		\end{align}
	Starting with one of the factors from the left side of (\ref{eq:cullina.1}), we have
	$$e^{q_2(z^{*2}-1)+q_1(z^*-1)} \leq e^{\frac{q_1}{2}z^*+\frac{\tau}{2}-q_2-q_1}\leq e^{\tau-q_1-q_2}\leq e^{\tau},$$
	where we used (\ref{eq:quad.eq}) to eliminate the $q_2z^{*2}$ term, applied the upper bound from (\ref{eq:bounds.z*}), and used $q_1, q_2 \geq 0$. From the lower bound in (\ref{eq:bounds.z*}), 
	$${z}^{*-2} \leq \frac{q_1^2}{\tau^2}+\frac{8q_2}{\tau} \leq \max \left\{\frac{2q_1^2}{\tau^2}, \frac{16q_2}{\tau}\right\}$$
	so $z^{*-\tau} e^{q_2(z^{*2}-1)+q_1(z^*-1)} \leq \zeta^\tau.$
\end{proof}

%
%




\bibliographystyle{abbrv}
\bibliography{graph_alignment1} 

\begin{thebibliography}{10}

\bibitem{abbe2017community}
E.~Abbe.
\newblock Community detection and stochastic block models: recent developments.
\newblock {\em The Journal of Machine Learning Research}, 18(1):6446--6531,
  2017.

\bibitem{almohamad1993linear}
H.~Almohamad and S.~O. Duffuaa.
\newblock A linear programming approach for the weighted graph matching
  problem.
\newblock {\em IEEE Transactions on pattern analysis and machine intelligence},
  15(5):522--525, 1993.

\bibitem{bhm}
S.~Banerjee, N.~Hegde, and L.~Massouli{\'e}.
\newblock The price of privacy in untrusted recommendation engines.
\newblock In {\em 2012 50th Annual Allerton Conference on Communication,
  Control, and Computing (Allerton)}, pages 920--927. IEEE, 2012.

\bibitem{barak2018nearly}
B.~Barak, C.-N. Chou, Z.~Lei, T.~Schramm, and Y.~Sheng.
\newblock (nearly) efficient algorithms for the graph matching problem on
  correlated random graphs.
\newblock {\em arXiv preprint arXiv:1805.02349}, 2018.

\bibitem{bayati2013message}
M.~Bayati, D.~F. Gleich, A.~Saberi, and Y.~Wang.
\newblock Message-passing algorithms for sparse network alignment.
\newblock {\em ACM Transactions on Knowledge Discovery from Data (TKDD)},
  7(1):3, 2013.

\bibitem{bollobas1982distinguishing}
B.~Bollob{\'a}s.
\newblock Distinguishing vertices of random graphs.
\newblock In {\em North-Holland Mathematics Studies}, volume~62, pages 33--49.
  Elsevier, 1982.

\bibitem{boucheron2013concentration}
S.~Boucheron, G.~Lugosi, and P.~Massart.
\newblock {\em Concentration inequalities: A nonasymptotic theory of
  independence}.
\newblock Oxford university press, 2013.

\bibitem{conte2004thirty}
D.~Conte, P.~Foggia, C.~Sansone, and M.~Vento.
\newblock Thirty years of graph matching in pattern recognition.
\newblock {\em International journal of pattern recognition and artificial
  intelligence}, 18(03):265--298, 2004.

\bibitem{cover2012elements}
T.~M. Cover and J.~A. Thomas.
\newblock {\em Elements of information theory}.
\newblock John Wiley \& Sons, 2012.

\bibitem{cullina2016improved}
D.~Cullina and N.~Kiyavash.
\newblock Improved achievability and converse bounds for erdos-r{\'e}nyi graph
  matching.
\newblock {\em ACM SIGMETRICS Performance Evaluation Review}, 44(1):63--72,
  2016.

\bibitem{cullina2017exact}
D.~Cullina and N.~Kiyavash.
\newblock Exact alignment recovery for correlated erdos renyi graphs.
\newblock {\em arXiv preprint arXiv:1711.06783}, 2017.

\bibitem{cullina2018partial}
D.~Cullina, N.~Kiyavash, P.~Mittal, and H.~V. Poor.
\newblock Partial recovery of erd$\backslash$h $\{$o$\}$
  sr$\backslash$'$\{$e$\}$ nyi graph alignment via $ k $-core alignment.
\newblock {\em arXiv preprint arXiv:1809.03553}, 2018.

\bibitem{czajka2008improved}
T.~Czajka and G.~Pandurangan.
\newblock Improved random graph isomorphism.
\newblock {\em Journal of Discrete Algorithms}, 6(1):85--92, 2008.

\bibitem{dai2018performance}
O.~E. Dai, D.~Cullina, N.~Kiyavash, and M.~Grossglauser.
\newblock On the performance of a canonical labeling for matching correlated
  erd\"os-r\'enyi graphs.
\newblock {\em arXiv preprint arXiv:1804.09758}, 2018.

\bibitem{ding2018efficient}
J.~Ding, Z.~Ma, Y.~Wu, and J.~Xu.
\newblock Efficient random graph matching via degree profiles.
\newblock {\em arXiv preprint arXiv:1811.07821}, 2018.

\bibitem{dym2017ds++}
N.~Dym, H.~Maron, and Y.~Lipman.
\newblock Ds++: A flexible, scalable and provably tight relaxation for matching
  problems.
\newblock {\em arXiv preprint arXiv:1705.06148}, 2017.

\bibitem{elmsallati2015global}
A.~Elmsallati, C.~Clark, and J.~Kalita.
\newblock Global alignment of protein-protein interaction networks: A survey.
\newblock {\em IEEE/ACM transactions on computational biology and
  bioinformatics}, 13(4):689--705, 2015.

\bibitem{fan2019spectral}
Z.~Fan, C.~Mao, Y.~Wu, and J.~Xu.
\newblock Spectral graph matching and regularized quadratic relaxations i: The
  gaussian model.
\newblock {\em arXiv preprint arXiv:1907.08880}, 2019.

\bibitem{fan2019spectral2}
Z.~Fan, C.~Mao, Y.~Wu, and J.~Xu.
\newblock Spectral graph matching and regularized quadratic relaxations ii:
  Erd$\backslash$h $\{$o$\}$ sr$\backslash$'enyi graphs and universality.
\newblock {\em arXiv preprint arXiv:1907.08883}, 2019.

\bibitem{feizi2019spectral}
S.~Feizi, G.~Quon, M.~Mendoza, M.~Medard, M.~Kellis, and A.~Jadbabaie.
\newblock Spectral alignment of graphs.
\newblock {\em IEEE Transactions on Network Science and Engineering}, 2019.

\bibitem{fishkind2012seeded}
D.~E. Fishkind, S.~Adali, H.~G. Patsolic, L.~Meng, V.~Lyzinski, and C.~E.
  Priebe.
\newblock Seeded graph matching.
\newblock {\em arXiv preprint arXiv:1209.0367}, 2012.

\bibitem{fogel2013convex}
F.~Fogel, R.~Jenatton, F.~Bach, and A.~d'Aspremont.
\newblock Convex relaxations for permutation problems.
\newblock In {\em Advances in Neural Information Processing Systems}, pages
  1016--1024, 2013.

\bibitem{ganassali2020tree}
L.~Ganassali and L.~Massouli{\'e}.
\newblock From tree matching to sparse graph alignment.
\newblock {\em arXiv preprint arXiv:2002.01258}, 2020.

\bibitem{kazemi2015growing}
E.~Kazemi, S.~H. Hassani, and M.~Grossglauser.
\newblock Growing a graph matching from a handful of seeds.
\newblock {\em Proceedings of the VLDB Endowment}, 8(10):1010--1021, 2015.

\bibitem{kazemi2015can}
E.~Kazemi, L.~Yartseva, and M.~Grossglauser.
\newblock When can two unlabeled networks be aligned under partial overlap?
\newblock In {\em 2015 53rd Annual Allerton Conference on Communication,
  Control, and Computing (Allerton)}, pages 33--42. IEEE, 2015.

\bibitem{kezurer2015tight}
I.~Kezurer, S.~Z. Kovalsky, R.~Basri, and Y.~Lipman.
\newblock Tight relaxation of quadratic matching.
\newblock In {\em Computer Graphics Forum}, volume~34, pages 115--128. Wiley
  Online Library, 2015.

\bibitem{korolev2010upper}
V.~Y. Korolev and I.~G. Shevtsova.
\newblock On the upper bound for the absolute constant in the berry--esseen
  inequality.
\newblock {\em Theory of Probability \& Its Applications}, 54(4):638--658,
  2010.

\bibitem{kuchaiev2010topological}
O.~Kuchaiev, T.~Milenkovi{\'c}, V.~Memi{\v{s}}evi{\'c}, W.~Hayes, and
  N.~Pr{\v{z}}ulj.
\newblock Topological network alignment uncovers biological function and
  phylogeny.
\newblock {\em Journal of the Royal Society Interface}, 7(50):1341--1354, 2010.

\bibitem{lacoste2006word}
S.~Lacoste-Julien, B.~Taskar, D.~Klein, and M.~I. Jordan.
\newblock Word alignment via quadratic assignment.
\newblock In {\em Proceedings of the main conference on Human Language
  Technology Conference of the North American Chapter of the Association of
  Computational Linguistics}, pages 112--119. Association for Computational
  Linguistics, 2006.

\bibitem{lyzinski2016graph}
V.~Lyzinski, D.~Fishkind, M.~Fiori, J.~Vogelstein, C.~Priebe, and G.~Sapiro.
\newblock Graph matching: Relax at your own risk.
\newblock {\em IEEE Transactions on Pattern Analysis \& Machine Intelligence},
  (1):1--1, 2016.

\bibitem{maccartney2008phrase}
B.~MacCartney, M.~Galley, and C.~D. Manning.
\newblock A phrase-based alignment model for natural language inference.
\newblock In {\em Proceedings of the conference on empirical methods in natural
  language processing}, pages 802--811. Association for Computational
  Linguistics, 2008.

\bibitem{malod2015graal}
N.~Malod-Dognin and N.~Pr{\v{z}}ulj.
\newblock L-graal: Lagrangian graphlet-based network aligner.
\newblock {\em Bioinformatics}, 31(13):2182--2189, 2015.

\bibitem{mossel2019seeded}
E.~Mossel and J.~Xu.
\newblock Seeded graph matching via large neighborhood statistics.
\newblock In {\em Proceedings of the Thirtieth Annual ACM-SIAM Symposium on
  Discrete Algorithms}, pages 1005--1014. SIAM, 2019.

\bibitem{narayanan2006break}
A.~Narayanan and V.~Shmatikov.
\newblock How to break anonymity of the netflix prize dataset.
\newblock {\em arXiv preprint cs/0610105}, 2006.

\bibitem{narayanan2009anonymizing}
A.~Narayanan and V.~Shmatikov.
\newblock De-anonymizing social networks.
\newblock In {\em 2009 30th IEEE symposium on security and privacy}, pages
  173--187. IEEE, 2009.

\bibitem{pedarsani2011privacy}
P.~Pedarsani and M.~Grossglauser.
\newblock On the privacy of anonymized networks.
\newblock In {\em Proceedings of the 17th ACM SIGKDD international conference
  on Knowledge discovery and data mining}, pages 1235--1243, 2011.

\bibitem{pittel1996sudden}
B.~Pittel, J.~Spencer, and N.~Wormald.
\newblock Sudden emergence of a giantk-core in a random graph.
\newblock {\em Journal of Combinatorial Theory, Series B}, 67(1):111--151,
  1996.

\bibitem{read1977graph}
R.~C. Read and D.~G. Corneil.
\newblock The graph isomorphism disease.
\newblock {\em Journal of graph theory}, 1(4):339--363, 1977.

\bibitem{rendl1994quadratic}
F.~Rendl, P.~Pardalos, and H.~Wolkowicz.
\newblock The quadratic assignment problem: A survey and recent developments.
\newblock In {\em Proceedings of the DIMACS workshop on quadratic assignment
  problems}, volume~16, pages 1--42, 1994.

\bibitem{sahni1976p}
S.~Sahni and T.~Gonzalez.
\newblock P-complete approximation problems.
\newblock {\em Journal of the ACM (JACM)}, 23(3):555--565, 1976.

\bibitem{santhanam2012information}
N.~P. Santhanam and M.~J. Wainwright.
\newblock Information-theoretic limits of selecting binary graphical models in
  high dimensions.
\newblock {\em IEEE Transactions on Information Theory}, 58(7):4117--4134,
  2012.

\bibitem{sason2016f}
I.~Sason and S.~Verd{\'u}.
\newblock $ f $-divergence inequalities.
\newblock {\em IEEE Transactions on Information Theory}, 62(11):5973--6006,
  2016.

\bibitem{scheinerman2011fractional}
E.~R. Scheinerman and D.~H. Ullman.
\newblock {\em Fractional graph theory: a rational approach to the theory of
  graphs}.
\newblock Courier Corporation, 2011.

\bibitem{schellewald2005probabilistic}
C.~Schellewald and C.~Schn{\"o}rr.
\newblock Probabilistic subgraph matching based on convex relaxation.
\newblock In {\em International Workshop on Energy Minimization Methods in
  Computer Vision and Pattern Recognition}, pages 171--186. Springer, 2005.

\bibitem{shirani2018typicality}
F.~Shirani, S.~Garg, and E.~Erkip.
\newblock Typicality matching for pairs of correlated graphs.
\newblock In {\em 2018 IEEE International Symposium on Information Theory
  (ISIT)}, pages 221--225. IEEE, 2018.

\bibitem{singh2008global}
R.~Singh, J.~Xu, and B.~Berger.
\newblock Global alignment of multiple protein interaction networks with
  application to functional orthology detection.
\newblock {\em Proceedings of the National Academy of Sciences},
  105(35):12763--12768, 2008.

\bibitem{umeyama1988eigendecomposition}
S.~Umeyama.
\newblock An eigendecomposition approach to weighted graph matching problems.
\newblock {\em IEEE transactions on pattern analysis and machine intelligence},
  10(5):695--703, 1988.

\bibitem{vogelstein2015fast}
J.~T. Vogelstein, J.~M. Conroy, V.~Lyzinski, L.~J. Podrazik, S.~G. Kratzer,
  E.~T. Harley, D.~E. Fishkind, R.~J. Vogelstein, and C.~E. Priebe.
\newblock Fast approximate quadratic programming for graph matching.
\newblock {\em PLOS one}, 10(4), 2015.

\bibitem{williams1992n}
K.~S. Williams.
\newblock The n th power of a 2$\times$ 2 matrix.
\newblock {\em Mathematics Magazine}, 65(5):336--336, 1992.

\bibitem{wright1971graphs}
E.~M. Wright et~al.
\newblock Graphs on unlabelled nodes with a given number of edges.
\newblock {\em Acta Mathematica}, 126:1--9, 1971.

\bibitem{wu2021settling}
Y.~Wu, J.~Xu, and S.~H. Yu.
\newblock Settling the sharp reconstruction thresholds of random graph
  matching.
\newblock {\em arXiv preprint arXiv:2102.00082}, 2021.

\bibitem{yartseva2013performance}
L.~Yartseva and M.~Grossglauser.
\newblock On the performance of percolation graph matching.
\newblock In {\em Proceedings of the first ACM conference on Online social
  networks}, pages 119--130. ACM, 2013.

\bibitem{zaslavskiy2008path}
M.~Zaslavskiy, F.~Bach, and J.-P. Vert.
\newblock A path following algorithm for the graph matching problem.
\newblock {\em IEEE Transactions on Pattern Analysis and Machine Intelligence},
  31(12):2227--2242, 2008.

\bibitem{zhang2010sapper}
S.~Zhang, J.~Yang, and W.~Jin.
\newblock Sapper: subgraph indexing and approximate matching in large graphs.
\newblock {\em Proceedings of the VLDB Endowment}, 3(1-2):1185--1194, 2010.

\bibitem{zhao1998semidefinite}
Q.~Zhao, S.~E. Karisch, F.~Rendl, and H.~Wolkowicz.
\newblock Semidefinite programming relaxations for the quadratic assignment
  problem.
\newblock {\em Journal of Combinatorial Optimization}, 2(1):71--109, 1998.

\end{thebibliography}



\end{document}